\documentclass[conference]{IEEEtran}
\usepackage{cite}

\usepackage{url}
\usepackage{amsmath,amsthm,amsfonts}
\usepackage{amssymb}
\usepackage{xcolor}
\usepackage{xspace}
\usepackage{algorithm}
\usepackage{algpseudocode}

\usepackage{graphicx}
\usepackage{caption}
\usepackage{subcaption}
\usepackage{multirow}

\usepackage{url}

\theoremstyle{plain}
\newtheorem{theorem}{Theorem}
\newtheorem{proposition}{Proposition}

\newtheorem{remark}{Remark}
\newtheorem{note}{Note}

\theoremstyle{definition}
\newtheorem{definition}{Definition}
\newtheorem{assumption}{Assumption}

\newcommand{\indicator}[1]{\mathbb{I}\left[#1\right]}

\newcommand{\realset}{\mathbb{R}}
\newcommand{\trainingset}{D}
\newcommand{\outputspace}{{\cal Y}}
\newcommand{\inputspace}{{\cal X}}

\newcommand{\versionspace}{{\cal W}}

\newcommand{\sign}{\text{sign}}
\newcommand{\inter}{\cap}
\newcommand{\ball}{{\cal B}}
\newcommand{\svm}{\text{\sc Svm}\xspace}
\newcommand{\convex}{{\cal C}}

\newcommand{\margin}{\gamma}
\newcommand{\dotprod}[2]{\langle #1, #2 \rangle}
\newcommand{\norm}[1]{\Vert #1 \Vert_2}
\newcommand{\vol}{\text{Vol}}
\newcommand{\cg}{\textbf{cg}}
\newcommand{\cc}{\textbf{cc}}

\newcommand{\algo}{{\cal A}}
\newcommand{\union}{\bigcup}
\newcommand{\compression}{{\cal S}}
\newcommand{\Dset}{{\cal D}}

\DeclareMathOperator*{\find}{\text{find }}

\newcommand{\change}[2]{\textcolor{red}{\---old : #1\---}\textcolor{cyan}{#2}}
\newcommand{\changeok}[2]{#2}

\newcommand{\activebpm}{{\tt Active-BPM}\xspace}
\newcommand{\activesvm}{{\tt Active-SVM}\xspace}

\usepackage{exscale}
\makeatletter
\newenvironment{equationsize}[1]{%
  \skip@=\baselineskip 
  #1%
  \baselineskip=\skip@ 
  \equation
}{\endequation \ignorespacesafterend} 
\makeatother

\makeatletter
\newenvironment{equationsize*}[1]{%
  \skip@=\baselineskip 
  #1%
  \baselineskip=\skip@ 
  \equation
}{\nonumber\endequation \ignorespacesafterend} 
\makeatother

 

\makeatletter
\newenvironment{alignsize*}[1]{%
  \skip@=\baselineskip 
  #1%
  \baselineskip=\skip@ 
  \start@align\@ne\st@rredtrue\m@ne
}{\endalign\ignorespacesafterend} 
\makeatother

\newcommand{\captionfonts}{\small}

\makeatletter  
\long\def\@makecaption#1#2{%
  \vskip\abovecaptionskip
  \sbox\@tempboxa{{\captionfonts #1: #2}}%
  \ifdim \wd\@tempboxa >\hsize
    {\captionfonts #1: #2\par}
  \else
    \hbox to\hsize{\hfil\box\@tempboxa\hfil}%
  \fi
  \vskip\belowcaptionskip}
\makeatother   

\input{header_app}

\begin{document}

\title{From Cutting Planes Algorithms to Compression Schemes and Active
Learning}

\author{\IEEEauthorblockN{Ugo Louche and Liva Ralaivola}
\IEEEauthorblockA{Qarma, LIF - CNRS, Aix-Marseille University.
Email: firstname.lastname@lif.univ-mrs.fr}
}

\maketitle

\begin{abstract}
Cutting-plane methods are well-studied localization
(and optimization) algorithms. We show that they provide a natural framework to perform machine
learning ---and not just to solve optimization problems posed by machine
learning--- in addition to their intended optimization use. In particular, they
allow one to learn sparse classifiers and provide good compression schemes.
Moreover, we show that very little effort is required to turn them into
effective active learning methods. This last property provides a generic way to
design a whole family of active learning algorithms from existing passive
methods. We present numerical simulations testifying of the relevance of
cutting-plane methods for passive and active learning tasks.
\end{abstract}


\section{Introduction} 
\label{sec:introduction}
We show that localization methods based on cutting planes
provide a natural framework to derive machine learning algorithms for
classification, both in the supervised learning framework and the active
learning framework. Our claim is that cutting plane algorithms, beyond their
optimization purposes, embed features that are beneficial for generalization
purposes. In particular a) under mild conditions, they may provide compression
scheme with a compression rate that is directly related to their aim at rapidly
finding a solution of the localization problem and b) the pivotal step of such
algorithms, namely, the querying step, may be slightly twisted so as to be
active-learning friendly.

In the present paper, we show that existing learning algorithms might
be revisited from the cutting planes point of view. \changeok{Not only will the
active learning SVM procedure of Tong and Koller~\cite{Tong02} fall into this class of
algorithms that might be reinterpreted as one instance of the framework we
describe but so will the so-called Bayes Point Machines~\cite{Herbrich01}  and
an active learning version of it.}{Not only might the active learning SVM procedure of
Tong and Koller~\cite{Tong02} be reinterpreted as an algorithm falling under the
framework we describe but so are the Bayes Point Machines~\cite{Herbrich01}, 
for which we will propose an active learning version of it.}

The problems we are interested in are linear classification problems.
Given a training sample $\trainingset\doteq\{(x_n,y_n)\}_{n\in[N]}$, with
$x_n\in\inputspace\doteq\realset^d,$
$y_n\in\outputspace\doteq\{-1,+1\},$  and $[N]\doteq\{1,\ldots,N\}$, we are looking for a classification
vector $w\in\inputspace$ that is an element of the version space
\begin{equation}
\versionspace_0(\trainingset)\doteq\left\{w\in\inputspace:y_n\langle
w,x_n\rangle\geq 0,\;n\in[N]\right\},
\end{equation}
 of $\trainingset$, i.e. the set of
vectors $w$ from $\inputspace$ such that the corresponding linear
predictors 
\begin{equation}f_w(x)\doteq\sign(\langle w,x\rangle)\end{equation}
make no mistake on
the training set $\trainingset.$  In order to render the exposition
clearer, we make the assumption that the training data are linearly
separable so that $\versionspace_0(\trainingset)$ is not empty. The case where
$\versionspace_0(\trainingset)=\emptyset$ can be tackled with usual machine
learning techniques ---e.g. the ``$\lambda$-trick'' and/or kernels
\cite{Freund99} \cite{Herbrich01}.

Also, for the sake of brevity, we may use $\versionspace_0$ instead of 
$\versionspace_0(\trainingset)$
and thus drop the explicit dependence on $\trainingset.$

With the relevant notation at hand, the problem we are interested in may be stated as:
\begin{align}
\label{eq:metaproblem}
\find  w\in\versionspace_0,
\end{align}
which might be simply
rewritten as the problem of solving a set of linear inequalities
\begin{align}
\label{eq:masterproblem}
\find  w\enskip
\text{ s.t. }&\left\{\begin{array}{l}
w\in\inputspace\\
y_n\left\langle w,x_n\right\rangle\geq 0,\; n\in[N].
\end{array}\right.
\end{align}
There
is a variety of methods in the optimization literature from as back as the 50's that
are available to solve such problems.
Among them, we may mention {\em (over-)relaxation} based
methods~\cite{goffin80MOR,Motzkin:1954:RML}, simplex-based algorithms and, of
course, the Perceptron algorithm and its numerous variants
\cite{Block62,Novikoff62,Rosenblatt58}. Localization methods based on {\em
cutting planes}, or, in short, cutting planes algorithms, are well-studied
algorithms, well-known to be very efficient to solve such problems. We will
show that, when used to solve~\eqref{eq:masterproblem}, i) they naturally
provide compression scheme algorithms~\cite{FloydW95MLJ}, and thus, {\em
learning} algorithms that embed features designed to ensure good generalization
properties and ii) they also set the ground for the development of new active
learning algorithms.

\subsection{Related Works}
Cutting-plane methods provide a family of optimizaton procedures that have
received some interest from the machine learning community~\cite{JoachimsFY09MLJ,FrancS09JMLR,TeoVSL10JMLR}.
However, they have mainly been considered as optimization methods to solve problems
such as those posed by support vector machines or, more generally, regularized risk functionals.
The more profound connection of these methods with {\em learning} algorithms, that is, procedures that are designed in a way 
to ensure generalization ability to the predictor they build (e.g. the Perceptron algorithm) has less been studied; this is
one of the peculiarities of the present paper to discuss this feature---to some extent, the work of~\cite{TeoVSL10JMLR},
which pinpoints how statistical regularization is beneficial for the stabilization of cutting-plane methods, skims over
this connection. Within the vast literature of active learning (see, e.g.~\cite{Settles12Active}),
we may single out a few contributions our work is closely related to; they 
share the common feature of focusing on/exploiting the geometry of the version space.
The query strategies proposed by \cite{golovin10} and \cite{gonen13} are  based
on multiple estimations of the volume of the (potential) version space, which,
when added together might be computationally expensive. In comparison, in the
active learning strategy we derive from the general cutting-plane approach, we
compute our queries from an approximated center of gravity of the version
space, which is computationally equivalent to a single volume estimation.
The work of \cite{balcan07}, who propose a margin-based query strategy
provide theoretical justifications of such strategies and gives insights
on the foundations the work of \cite{Tong02} hinges on.
 Our contribution is to show how the cutting planes
literature and its accompanying worst-case convergence analyzes may
give rise to theoretically supported query strategies that do not have to
hinge on margin-based arguments. \changeok{}{To some extent, our work has
connections with \emph{uncertainty-based active learning} (see, e.g. \cite{lewis1994})
which advocates to query the points whose class is the most uncertain; 
our approach may be re-interpreted as a theoretically motivated
uncertainty measure based on the volume reduction of the version
space.}

\subsection{Outline}
The paper is structured as follows. Section \ref{sec:background} provides some
background to cutting planes methods and their possible application to
learning. Section \ref{sec:results} further explores the connections
between cutting planes and learning algorithms and then provides a way to turn
cutting planes methods into an active learning algorithms. Section
\ref{sec:simulations} reports empirical results for algorithms derived from our
argumentation on the relevance of cutting plane methods to machine learning.


\section{Background}
\label{sec:background}
In this section, we first recall the general form of a cutting plane algorithm
to solve a localization problem. We then specialize this algorithm to the case
where the convex space into which we want to find a point is the version
space associated to training set $D$.
Finally, in order for the reader to get a taste on how cutting planes
algorithms give rise to {\em learning algorithms}, i.e.
algorithms that embed features, namely, they define {\em compression schemes} with targeted small compression size, that are beneficial for
generalization.

\subsection{Vanilla Localization Algorithm with Cutting Planes}
In order to solve a problem like
$$\find w \in\convex,$$
for $\convex$ some closed convex set, a localization algorithm
based on cutting planes works as follows (see also the synthetic depiction in
Algorithm~\ref{tab:vanilla_cp}) \cite{Kelley60SIAM}.
The algorithm maintains and iteratively refines (i.e. reduces) a closed convex
set $\convex^t$ that is known to contain $\convex$. From $\convex^t$ a {\em
query point} is computed ---there are several ways to compute such query
points; we will mention some when specializing localization methods to the
specific problem of finding a point in the version space later on--- which
leads to two possible options: either a) $w^t$ is in $\convex$ and the tackled
problem is solved or b) $w^t\not\in\convex$. In the latter case, a so-called
{\em cutting plane oracle} is queried with $w^t$ upon which it returns the
parameters $(a_t,b_t)$ of the hyperplane $\{z:\langle
a_t,z\rangle = b_t\}$ such that this hyperplane separates $w_t$ from  $\convex$,
i.e., $\forall w\in\convex, \langle a_t,w\rangle > b_t$ and $\langle
a_t,w_t\rangle \leq b_t$. The hyperplane is used to reduce $\convex^t$ into
$\convex^{t}\inter\{w:\langle a_t,w\rangle > b_t\}$ (which still contains
$\convex$). For the specific problem~\eqref{eq:masterproblem} of finding a point in the version space,
the cutting planes rendered by the oracle will be such that $b_t=0$.
 
\begin{algorithm}
\caption{Classical Cutting Plane Algorithm for the localization of
$w\in\convex$. \label{tab:vanilla_cp}}
\begin{algorithmic}[1]
\Ensure $w\in\convex$
\State compute $\convex^0$, such that $\convex^0\supset\convex$ and 
$\convex^0$ is convex and closed.
\State $t\leftarrow 0$
\Repeat
	\State Compute {\em query point} $w^t$ in $\convex^t$
	\State Ask the {\em cutting plane oracle} whether $w^t\in\convex$ 
	\If { $w^t \notin \convex$ }
		\State Receive a cutting plane $(a_t,b_t)$
		\State $\convex^{t+1}\leftarrow\convex^{t}
					\inter\left\{x:\langle a_t,x\rangle > b_t
					\right\}$
		\State $t\leftarrow t+1$
	\EndIf
\Until{$w^t \in \convex$}
\State \Return $w^t$
\end{algorithmic}
\end{algorithm}
\begin{algorithm}
\caption{The Cutting Plane approach instantiated to the problem of finding a
point from the version space of $D$. \label{tab:versionspace_cp}}
\begin{algorithmic}[1]
\Ensure $w$ solution of Problem \eqref{eq:trueproblem}
\State $\convex^0\leftarrow\ball$
\State $t\leftarrow 0$
\Repeat
	\State $w^t \leftarrow \Call{Query}{\convex^t}$ \label{line:query} \Comment
	Compute {\em query point} $w^t$ in $\convex^t$
	\If {$w^t \notin \versionspace$}
		\State $n_t \leftarrow \Call{Pick}{\convex^t, w^t}$ \label{line:n_t}
		\Comment pick a cutting plane index 
		\State $\convex^{t+1}\leftarrow\convex^{t}
		\inter\left\{z:y_{n_t}\dotprod{z}{x_{n_t}} > 0 \right\}$
		\State $t\leftarrow t+1$
	\EndIf
\Until{$w^t \in \versionspace$}
\State \Return $w^t$
\end{algorithmic}
\end{algorithm}

\subsection{Cutting Planes to Localize a Point in the Version Space}
Note that problem~\eqref{eq:masterproblem} is scale-insensitive: if
$w\in\versionspace_0$, then $\lambda w\in\versionspace_0$ as well for
any $\lambda>0.$ In order to get rid of this degree of freedom {\em and}
to make the use of cutting planes algorithms possible (they require the sets $\convex^t$ to be bounded), we will restrict
ourselves to  finding  a solution vector $w^*$ both in
 $\versionspace_0$ and in the unit ball 
\begin{equation}
\ball\doteq\left\{w\in\inputspace:\|w\|\leq 1\right\}.
\end{equation}
In other words, we will be looking for $w^*$ in the constrained version space
\begin{equation}
\versionspace\doteq\versionspace_0\inter\ball,
\end{equation}
and the problem we face is therefore:
\begin{align}
\label{eq:trueproblem}
\find  w\enskip
\text{ such that }&\left\{\begin{array}{l}
w\in\ball\\
y_n\langle w,x_n\rangle\geq 0,\;n\in[N]
\end{array}\right.
\end{align}

In the case of Problem~\eqref{eq:trueproblem}, the localization algorithm
described earlier translates into the one given in
Algorithm~\ref{tab:versionspace_cp}. The following changes might be observed
when comparing with Algorithm~\ref{tab:vanilla_cp}: $\convex^0$ is now
initialized to $\ball$, the unit ball, and the cutting planes
are picked among the hyperplanes ---i.e. the points of $\trainingset$---
defining the version space.

\subsection{Query Point Generation}
In both Algorithm~\ref{tab:vanilla_cp} and Algorithm~\ref{tab:versionspace_cp},
the strategy to compute a query point is left unspecified. There actually exist
many ways to compute such query points, but they all aim at a query point which
calls for a cutting plane that will divide the current
enclosing convex set $\convex^t$ in the most stringent way. It turns out that
such guarantee might be expected when the query point is as close as possible
to the `center' of $\convex^t$, so that the volume of $\convex^t$ is reduced with
a positive factor ---just as in the well-known bisection method, where the factor is 1/2. The center
of $\convex^t$ is not defined in a unique way, but for the most popular query
methods, it may refer to: a) the center of gravity of $\convex^t$, b) the
center of the largest ball inscribed in $\convex^t$, which is called the {\em Chebyshev center} or
c) the analytic center, which we will not discuss further (the interested
reader may refer to~\cite{Nesterov95} for further details). We may mention three
things regarding the center of gravity: i) it is NP-hard\footnote{To be precise, it is actually \#P-hard.}
 to exactly compute
the center of gravity of a convex set in an arbitrary $n$-dimensional space
even though some practical approximation algorithms exist; ii) it is the query
point that comes with the best guarantees in terms of convergence speed of the
cutting plane method~\cite{DabbeneSP10SIAM};  iii) the center of gravity of a
polytope is precisely the point that is looked for in the case of the
theoretically founded Bayes Point Machines of \cite{Herbrich01}.

\section{Results}
\label{sec:results}
This section is devoted to some algorithmic results that can be
obtained when analyzing the behavior of cutting-plane methods for the localization
of a point in the version space.

\subsection{Cutting Planes Provide Sample Compression Schemes}
Let $\Dset\doteq\union_{n=1}^{\infty}(\inputspace\times\outputspace)^n$ be the set 
of all finite training samples made of pairs from $\inputspace\times\outputspace$.
In short, sample compression schemes~\cite{FloydW95MLJ} are learning algorithms
$\algo:\Dset\rightarrow \outputspace^{\inputspace}$ that are associated with a compression function $\compression:\Dset\rightarrow\Dset$ so that, given any training sample $D$, 
we have $\algo(D)=\algo(\compression(D)).$
Sample compression schemes are especially interesting when the size $|\compression(D)|$
of the
compression set $\compression(D)$ is small. Indeed, generalization guarantees that come with these
procedures say that the generalization error of $f_D\doteq\algo(D)$ is, with high probability
(over the random draw of training set $D$ according to an unknown and fix distribution) bounded from above by something like
\begin{equationsize}{\footnotesize}
\label{eq:cs_generalization}
\frac{1}{N-|\compression(D)|}\sum_{n=1}^N\indicator{f_D(x_n)\neq y_n}+\mathcal{O}\left(\sqrt{\frac{1}{N-|\compression(D)|}}\right)
\end{equationsize}
 (see \cite{FloydW95MLJ,graepel05pacbayesian} for a precise statement of the bound).
Among the most well-known learning compression schemes, we find the Perceptron and the
Support Vector Machines. 

We claim that Algorithm~\ref{tab:versionspace_cp}, which finds a point in the version space using cutting planes, may be a compression scheme.
\begin{proposition}
\label{prop:cs}
\changeok{If the method to find a query point (line~\ref{line:query},
Algorithm~\ref{tab:versionspace_cp}) 
is deterministic and the procedure
implemented to choose the next $n_t$ (line~\ref{line:n_t}) is deterministic and
only depends on $\convex^t$ and 
$w^t$ then Algorithm~\ref{tab:versionspace_cp}
is a sample compression scheme.Let $\convex$ and $\convex'$ two equivalent
representations of the same version space. }{If $\Call{Query}{\convex^t}$
(line~\ref{line:query}, Algorithm~\ref{tab:versionspace_cp}) and
$\Call{Pick}{\convex^t,w^t}$ (line~\ref{line:n_t}) are both deterministic then
Algorithm~\ref{tab:versionspace_cp} is a sample compression scheme.}
\end{proposition}
\begin{proof}
If the compression set is made of the training examples that define the cutting planes, this result 
is a direct consequence of the structure of
Algorithm~\ref{tab:versionspace_cp}. A proof by induction that essentially hinges on the 
fact that, at each iteration $t$, the next query point is deterministically computed from $\convex^t$ (only) gives the result.
\end{proof}
A few observations can be made. First, the learning algorithm obtained with the
assumptions of Proposition~\ref{prop:cs} is a {\em process sample compression
scheme}, that is, even if we interrupt the learning before convergence has
occurred, running the algorithm on the partial compression scheme obtained so
far gives exactly the same predictor. Second, it is obviously an aim to have
fast convergence of the localization procedure, where fast convergence means
few iterations of the cutting-plane procedure. This directly translates into
the idea of finding a point in the version space that is expressed as a
combination as few vectors as possible, which, by~\eqref{eq:cs_generalization}, is very beneficial for
generalization purposes. Later, we will see that there are settings for cutting-plane methods that 
come with guarantees on the number of iterations, and therefore on
$|\compression(D)|$, to reach convergence.

\begin{algorithm}[t]
\begin{algorithmic}[1]
\Ensure Problem~\eqref{eq:trueproblem}
\State $\convex^0\leftarrow\ball$
\State $t \leftarrow 1$, $w^0 \leftarrow 0$, $\tilde{w}^0 \leftarrow 0$
\Repeat
	\State $\tilde{w}^{t} \leftarrow \Call{Perceptron}
	{\tilde{w}^{t-1}, x_{n_0}, \cdots, x_{n_t}}$
	\State $w^t \leftarrow {\tilde{w}^t}/{\norm{\tilde{w}^t}}$
	\If {$w^t \notin \versionspace$}
		\State Pick a cutting plane index $n_t$ \label{line:perc:pick}
		\State $\convex^{t+1}\leftarrow\convex^
		{t}\inter\left\{z:y_{n_t}\dotprod{z}{x_{n_t}}
		\geq 0\right\}$
		\State $t\leftarrow t+1$
	\EndIf
\Until{$w^t \in \versionspace$}
\State \Return $w^t$
\\
\Function{Perceptron}{$w^{start}, x_{n_0}, \cdots , x_{n_N}$}
\State $t\leftarrow 0$
\State $w^0 \leftarrow w^{start}$
\While {$\exists n_i: \dotprod{w^t}{x_{n_i}} < 0$} \label{line:perc:while}
	\State $w^{t+1} \leftarrow w^t + x_{n_i}$ \label{step:percUpdate}
	\State $t\leftarrow t+1$
\EndWhile
\State \Return $w^t$
\EndFunction
\end{algorithmic}
\caption{Top : A Perceptron-based localization algorithm for the case of
problem~\eqref{eq:trueproblem}. Bottom : The slightly modified
perceptron algorithm for compression scheme.\label{tab:perceptron_cp}}
\end{algorithm}

\subsection{Perceptron-based Localization Algorithm}

One of the simplest ways to compute a query point $w^t$ for Algorithm~\ref{tab:versionspace_cp} is to run Rosenblatt's Perceptron algorithm
\cite{Rosenblatt58} at each step and query the normalized solution $w^t =
{\tilde{w}^t}/{\norm{\tilde{w}^t}}$. Intuitively, we may expect $\tilde{w}^{t+1}$
to be `close' to $\tilde{w}^t$ because $\convex^{t+1}$ is essentially the intersection
of $\convex^t$ with a cutting plane and much of the geometry of $\convex^t$
might be preserved. According to this intuition,
$\tilde{w}^t$ should be a good starting point for the Perceptron algorithm
to be run and to have it output $\tilde{w}_{t+1}$. Algorithm \ref{tab:perceptron_cp}
implements that idea, and reuses the last query point as an initialization
vector for the Perceptron to compute the next query point.
\changeok{}{Additionally, note that for Algorithm \ref{tab:perceptron_cp} to match Proposition \ref{prop:cs} a little
technicality is needed: we require that datapoints are selected in the
lexicographical order\footnote{This is an arbitrary choice and any total order over $\realset^d$ can
be used instead} when multiple choices are possible (e.g.
line \ref{line:perc:pick} and \ref{line:perc:while}).} It turns out this simple querying procedure enjoys the same convergence rate than a regular Perceptron,
with the added empirically observed benefit of providing stronger compression
(see Section~\ref{sec:simulations} for empirical results).
\begin{proposition} \label{prop:ActiveNovikoff}
Consider Problem~(\ref{eq:trueproblem}) and let $\gamma$ be the radius of the largest inscribed sphere in
$\versionspace$. Define $M$ the number of Perceptron updates performed by the
Perceptron-based Localization Algorithm~\ref{tab:perceptron_cp} (i.e. $M$ is
the number of times line~\ref{step:percUpdate} of {\sc Perceptron}() of
Algorithm~\ref{tab:perceptron_cp} is executed).
Then the following holds: $M \leq
1/{\gamma^2}$.
\end{proposition}

\begin{proof}
We recall that the usual definition of the margin of $\trainingset$ is
$\min_{x \in \trainingset} \dotprod{w^*}{x}$ and note that $\gamma$ is related
to it since $\forall n\in[N],\; \dotprod{w^*}{x_n} / \Vert x_n \Vert_2 \geq
\gamma$. Let $\compression \doteq \lbrace a_1, \ldots a_M \rbrace$ be the sequence of points used to perform 
Perceptron updates across a complete execution of Algorithm
\ref{tab:perceptron_cp}. Thus, $\compression$ is a sequence from $\trainingset$ (with
possible duplicates) and $w^*$ achieves a margin at least $\margin$ with
all points in $\compression$. From \cite{Block62,Novikoff62} we know that the number $M$ of Perceptron updates
on any arbitrary sequence linearly separable with margin $\margin$ is no more than
$1/{\margin^2}$. Since we use $w^t$ as a starting point to compute $w^{t+1}$, the
execution of the cutting-plane algorithm is tied to the execution of the
Perceptron algorithm on $\compression$. Therefore, there is less than
$1/{\margin^2}$ Perceptron updates during the execution of the algorithm.
Alternatively, $\vert \compression \vert \leq 1/{\margin^2}$ since all points
in $\compression$ correspond to a Perceptron update, thus a mistake.
\end{proof} 

On a side note, the same argument can be applied to obtain  
similar results with most Perceptron-like learning procedures (see for instance \cite{Li2002,Crammer2006}).

\subsection{Center of Gravity and Approximations}

\begin{figure}[t]
\centering
\includegraphics[width=0.99\linewidth]{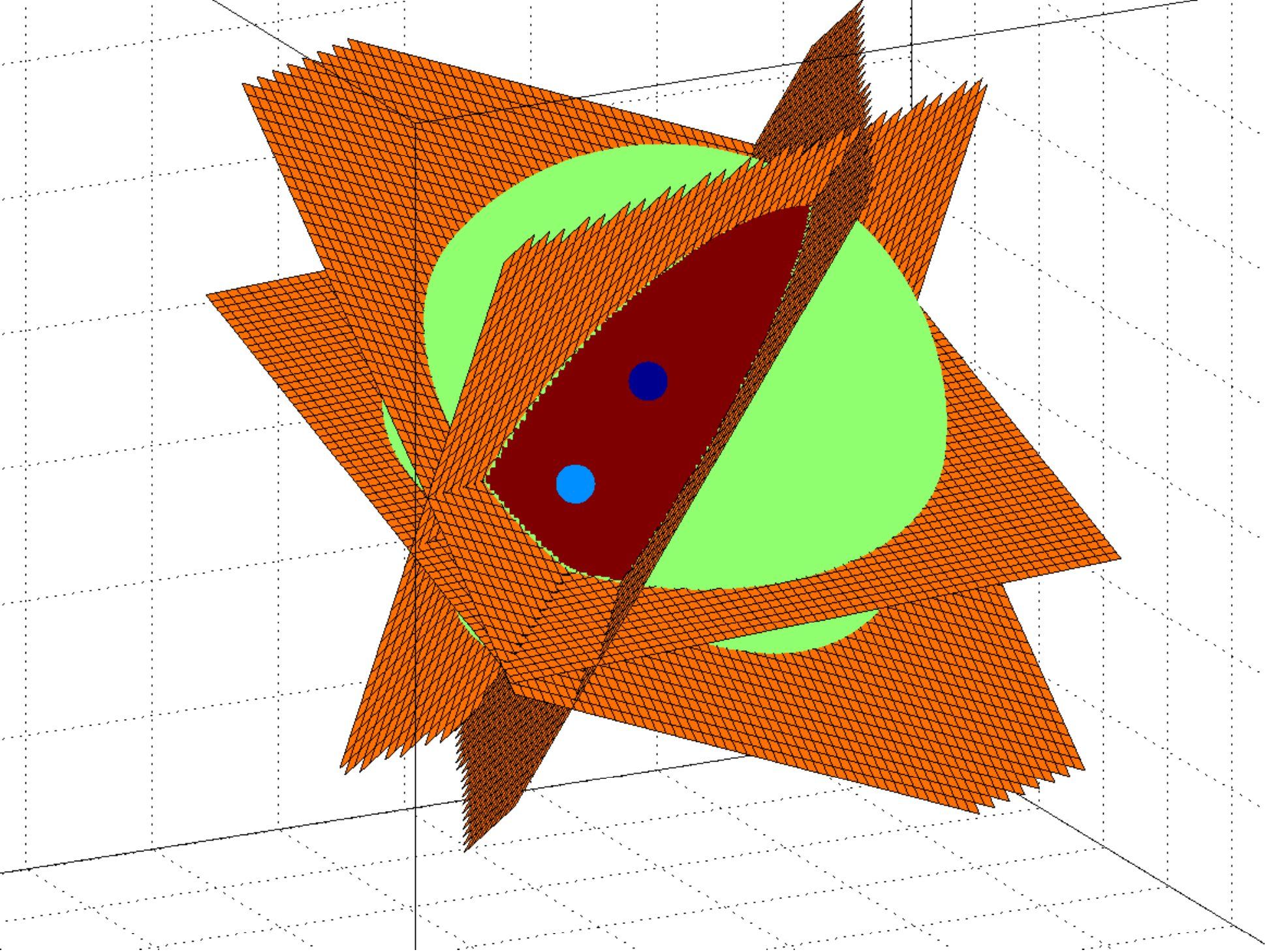}
\caption{An Example of version space where the Chebyshev Center
(light blue) is a bad approximation of the gravity center (dark blue).}
\end{figure}

The question of computing a query point $w^t$ is of central importance in
cutting-plane localization algorithms. As we have seen, a simple Perceptron 
can already yield interesting computational results for that matter.
A more assiduous analysis of this question can be conducted by looking at the
volume reduction $\vol(\convex^{t+1})/\vol(\convex^t)$ of $\convex^t$ from one iteration
to the next.
The notion of center of gravity is going to be pivotal to this end.
\begin{definition}[Center of Gravity]
Let $\convex$ be a closed set in $\realset^n$. The \emph{center of gravity} (CG)
$\cg(\convex)$  of $\convex$ is defined by as 
$ \cg(\convex)\doteq \int_{\convex} z dz/\int_{\convex} dz.$
\end{definition}

The center of gravity is deeply tied to the volume of $\convex^t$ and plays a
central role in devising cutting-plane algorithms for which
the volume reduction $\vol(\convex^{t+1})/\vol(\convex^t)$ is the largest. Theorem \ref{th:magic}
reports one of the most fundamental property of the center of gravity (see
\cite{grunbaum60, Newman65, Levin65, Boyd08})
\begin{theorem}[Partition of Convex bodies] \label{th:magic}
Let $\convex \in \realset^d$ a convex body of center of gravity $\cg(\convex)$
and $h$ a hyperplane such that $\cg(\convex) \in h$. Thus, $h$ divide $\convex$ in two
subsets $\convex_1$ and $\convex_2$ and the following relations hold for $i =
1,2$: $\vol(\convex_i) \geq e^{-1} \vol(\convex)$
\end{theorem}

The center of gravity method proposed by \cite{Newman65, Levin65}
consists in querying $w^t = \cg(\convex^t)$ and typically have a very
fast convergence rate as the version space is almost halved at each step. More
precisely, a direct consequence of Theorem~\ref{th:magic} is that the volume of
$\convex^t$ is bounded by $\vol(\convex^t) \leq (1 - 1/e)^t
\vol(\convex^0)$. 
However, computing the center a gravity is hard, making the
center of gravity method impractical. Instead, one has to consider structural or
numerical approximations to the center of gravity.

\begin{definition}[Chebyshev's Center]
Let $\convex$ a set in $\realset^n$. \emph{Chebyshev's center} (CC) of $\convex$,
$\cc(\convex)$ is the center of the largest inscribed ball in $\convex$:
$$\cc(\convex) = \arg\min_{\hat{z}}\max_z \norm{z - \hat{z}}^2.$$
\end{definition}
Chebyshev's center is used as a computationally efficient
approximation of the center of gravity for cutting-plane algorithms since the
late 70's \cite{Elzinga75} (see, e.g. \cite{Boyd04} for a linear formulation of the 
problem). Unfortunately, the interesting property of Theorem \ref{th:magic} does
not carry over with Chebyshev's center. One problem in machine learning related to
 Chebyshev's center is the extensively studied Support Vector Machine (\svm)
\cite{Vapnik95} defined as :
\begin{align}
\label{eq:svm}
\min_w \frac{1}{2} \Vert w \Vert_2^2 \enskip
\text{ s.t. }&\left\{\begin{array}{l}
w\in\inputspace\\
y_n\left\langle w,x_n\right\rangle \geq 1,\; n\in[N].
\end{array}\right.
\end{align}
A notable property of the \svm is that
its solution $w_{\svm}$ is closely related to the center of the largest
inscribed ball in $\versionspace$  and is an
approximation of the center of gravity \cite{Herbrich01}. Indeed,
$w_{\svm}$ is actually a rescaled Chebyshev's center \cite{Tong02}
\cite{Herbrich01}.

On the other hand, numerical approximations aim at
finding a point that is in the close neighborhood of the center of gravity.
One of the contributions of this paper is to give a generalized version of
Theorem \ref{th:magic} for approximations of the center of gravity, thus laying
a theoretical justification for these methods.

\begin{theorem}[Generalized Partition of Convex bodies] \label{th:generalized}
Let $\convex$ be a closed convex body in $\realset^d$ and $\cg(\convex)$ its center of gravity.
Let $h_{x}$ a hyperplane of normal vector $x$, $\norm{x} = 1$
and define the upper (resp. lower) partition $\convex^+$ (resp. $\convex^-$) of
$\convex$ by $h_{x}$ as 
\begin{align*}
\convex^+ \doteq \convex \cap \left\{ w \in \realset^d:
\dotprod{x}{w} \geq 0 \right\}\\
\convex^- \doteq \convex \cap \left\{ w \in \realset^d:
\dotprod{x}{w} < 0 \right\}.
\end{align*}
 The following holds true: if $\cg(\convex) + \Lambda x
\in \convex^+$ then 
\[
\vol(\convex^+)/\vol(\convex) \geq e^{-1} (1 - \lambda)^{d},
\]
where
\begin{align*}
\Lambda = \lambda \Theta_d \frac{\vol(\convex) H_{\convex^+}}{R^{d}
H_{\convex^-}},
\end{align*}
with $\lambda \in \realset$ an arbitrary real, $\Theta_d$ a constant depending
only on $d$, $R$ the radius of the $(d-1)$-dimensional ball $B$ of
volume $\vol\left[ B \right] \doteq \vol\left[ \convex \cap
\left\{ w \in \realset^d :
\dotprod{x}{w} = 0 \right\}  \right]$ and $H_{\convex^+} = \max_{a \in
\convex^+} a^Tx$ (resp.
$H_{\convex^-} = \min_{a \in \convex^-} a^T x$)
\end{theorem}
\begin{proof}
The proof is a (non-trivial) extension of Grunbaum's one for
Theorem~\ref{th:magic} \cite{grunbaum60}. Due to space restriction, we
cannot expose it here in full and refer the interested reader to
\emph{\url{http://pageperso.lif.univ-mrs.fr/~ugo.louche/paper/activeCPSuppl.pdf}}
\end{proof}

Theorem \ref{th:generalized} extends Theorem~\ref{th:magic} to the
situation when an approximation of the center of gravity is considered; 
it reduces to Theorem~\ref{th:magic} when applied to the very center of
gravity. This is to the best of our knowledge
the first result of this kind and this is a result that is of
its own interest, wich may benefit to many fields of computer science. 
Here, the purpose of Theorem~\ref{th:generalized} is essentially to validate
the use of approximations of the center of gravity $\cg(\convex)$ in the procedures at hand, which is
inevitable due to the complexity of exactly finding this point. We will more
precisely use it in two occasions: a) for center-of-gravity-based compression
scheme methods and b) in the active learning setting (see below).

\subsection{Active Learning with Cutting Planes} 
An interesting situation of learning is that of active learning
when the algorithm is presented with unlabelled data and
it has to query for the labels of the training points that carry
the most information to build a relevant decision boundary. Given 
a volume $\convex$ inside which a good classifier $w^*$ for the classification
task at hand is known to lie, the amount of information carried 
 by a labeled training point $(x,y)$ (where $y$ has been queried) might be 
 for instance measured by how $(x,y)$ can be used to identify within
 $\convex$ an (hopefully small) volume $\convex'\subseteq\convex$ where $w^*$ lives. Termed otherwise,
 the amount of information provided by $(x,y)$ might be measured as the volume 
 reduction induced by the knowledge of $(x,y)$: this is exactly the
 type of information cutting-plane methods build upon. We take advantage
 of this philosophy shared by active learning methods and cutting-plane algorithms
 to argue it is easy to transform a cutting-plane algorithm into an active
 learning method. Based on the idea of maximum volume reduction, the question
 to address is simply that of identifying a training pattern $x$ in $\trainingset$ 
 such that, independently of the label it might receive, is guaranteed to
 define a cutting hyperplane of equation $\langle x,w\rangle=0$ that intersects
 the current convex $\convex$ in a controlled way. To do so, a typical good query point
 is one that is as close as possible to the `center' of $\convex$, where center 
 may have the few meanings discussed above (cf. center of gravity, Chebyshev's
 center). The algorithm given in Table~\ref{tab:ActiveCP} is a generic active
 learning algorithm that is based on the classical cutting-plane approach.
\begin{algorithm}[t]
\begin{algorithmic}[1]
\State $\convex^0\leftarrow\ball$
\State $t\leftarrow 0$
\Repeat
	\State $w^t\leftarrow\text{center}(\convex^t)$
	\State $x_{n_t},y_{n_t} \leftarrow$ \Call{Query}{$\convex^t, \trainingset$} 
	\If {$y_{n_t}\dotprod{w^t}{x_{n_t}} < 0$}					
		\State $\convex^{t+1}\leftarrow\convex^{t}
		 \inter\left\{z:y_{n_t}\dotprod{z}{x_{n_t}}
		 \geq 0 \right\}$
		\State $t\leftarrow t+1$
	\EndIf
\Until{$\convex^t$ is small enough}
\State \Return $w^t$
\\
\Function{Query}{$\convex, \trainingset$}
 \State Sample $M$ points $s_1, \ldots s_M$ from $\convex$ 
\State ${\bf g}\leftarrow {\sum_{k=1}^M s_k}/{M}$
\State $x \leftarrow \arg\min_{x_i \in \trainingset} \dotprod{{\bf g}}{x_i}$
\State $y\leftarrow$ get label from an expert
\State \Return $x,y$
\EndFunction
\end{algorithmic}
\caption{Top: a generic cutting-plane active learning procedure; $w^t$ is computed as the `center' of $\convex^t$ ---center my refer to the center of gravity of the Chebyshev center. Bottom: a possible
implementation of {\sc Query()}: sampling strategies are given in, e.g.,
 \cite{Herbrich01,Lovasz2006,Kannan2012}.\label{tab:ActiveCP}}
\end{algorithm}

Making active learning algorithms from cutting-plane methods
is a route that has been taken by \cite{Tong02}, even though
the connection with cutting-plane algorithms was not clearly identified.

Being able to approximate the center of gravity of a convex polytope is
pivotal for the design of active learning strategies.
It is interesting to note that in the recent years, methods have been devised to uniformly sample
from the version space such as the \emph{Hit-and-Run} algorithm
of \cite{Lovasz2006} or a billiard algorithm of~\cite{Rujan97}. More
recently, the \emph{Dikin Walk} algorithm of~\cite{Kannan2012} provided a  strongly
polynomial algorithm for approximate uniform sampling over the version space
while the \emph{Expectation Propagation} method of \cite{Minka13} gave
a Bayesian interpretation of billiard algorithms.
Notably, these methods have been successfully used with cutting planes for
active Boosted Learning \cite{Trapeznikov2011}. Another practical approach
we should mention is the one proposed in \cite{Herbrich01} that
consists in repeatedly running a Perceptron over a permutation of the training
set:
in the active learning setting, the number of labeled points available is just
too low to produce interesting approximation of the center of gravity with
this method. 

A by-product of our active learning procedure is that we now solve a Bayes Point Machine (BPM)
problem \cite{Herbrich01} at each step $t$ by finding the center of gravity of
the current convex body $\convex^t$.
Therefore, we can turn our active learning procedure into a full active learning
algorithm---that we dub \texttt{Active-BPM}---for free by using the center of
gravity for classification.
Note that this is one of many possible instantiations of our procedure, which is nonetheless of
interest as it is the BPM-counterpart the \texttt{Active-SVM }algorithm of Tong
and Koller~\cite{Tong02}.

In conclusion,
Theorem~\ref{th:generalized} provides a general guideline to systematically
query the training point that comes with the best volume reduction guarantees.
This is a theoretically sound and viable strategy for active learning
that comes with a theoretical bound on the induced volume reduction, the lack of which was an essential limit 
of the Chebyshev's center-based method of \cite{Tong02}.


\section{Numerical Simulations}
\label{sec:simulations}

Here, we present some empirical simulations based on the algorithms described
throughout this paper in both passive and active learning settings.

\subsection{Synthetic Data and Perceptron-based Localization Algorithm}
We generate a toy dataset of $1,000$ $2$-dimensional datapoints. 
Each point is uniformly drawn on a $20$-by-$20$ square centered at the origin. 
We label this dataset according to a classifier $w^*$ uniformly drawn over 
the unit circle. In order to have only positive labels, negative examples 
are reflected through the origin. We then enforce a minimal margin 
$\gamma$ by pruning examples $x_i$ for which $\dotprod{w^*}{x_i} < \gamma$.
This last modification allows us to have some control over the size of the
version space $\versionspace$. The downside of this is that we no longer 
have exactly $1,000$ datapoints (though during our experiments we noted 
that the size of the dataset stays mostly the same for reasonable margin values).

For these experiments, we use the Perceptron-based Localization algorithm 
(Algorithm \ref{tab:perceptron_cp}). We implement it with three different 
oracle strategies for selecting cutting planes. The first strategy 
(which we call \emph{Largest Error}) picks the cutting plane with the 
lowest margin. The second one (\emph{Smallest Error}) picks the cutting
plane with the highest negative margin, that is to say points that
are incorrectly classified but close to the decision boundary.
Finally, the third one (\emph{Random Error}) simply picks a cutting plane with negative margin
at random. It should also be noted that our instantiation of the Perceptron
algorithm picks the update vector that
realizes the lowest margin for its internal update---line (\ref{step:percUpdate}) of {\sc Perceptron}() in Algorithm
\ref{tab:perceptron_cp}.
This is mostly an arbitrary choice and we only mention it for the sake of repoducibility.

The first experiment consists in a single run over a dataset of margin $\gamma
= 0.1$. We monitor both the number of cutting planes generated and the
number of internal Perceptron updates for each cutting plane. The presented
results are averaged over $1,000$ runs.

\begin{figure*}[t]
\centering
\begin{subfigure}{0.38\textwidth}
\centering
\includegraphics[width=1\linewidth]{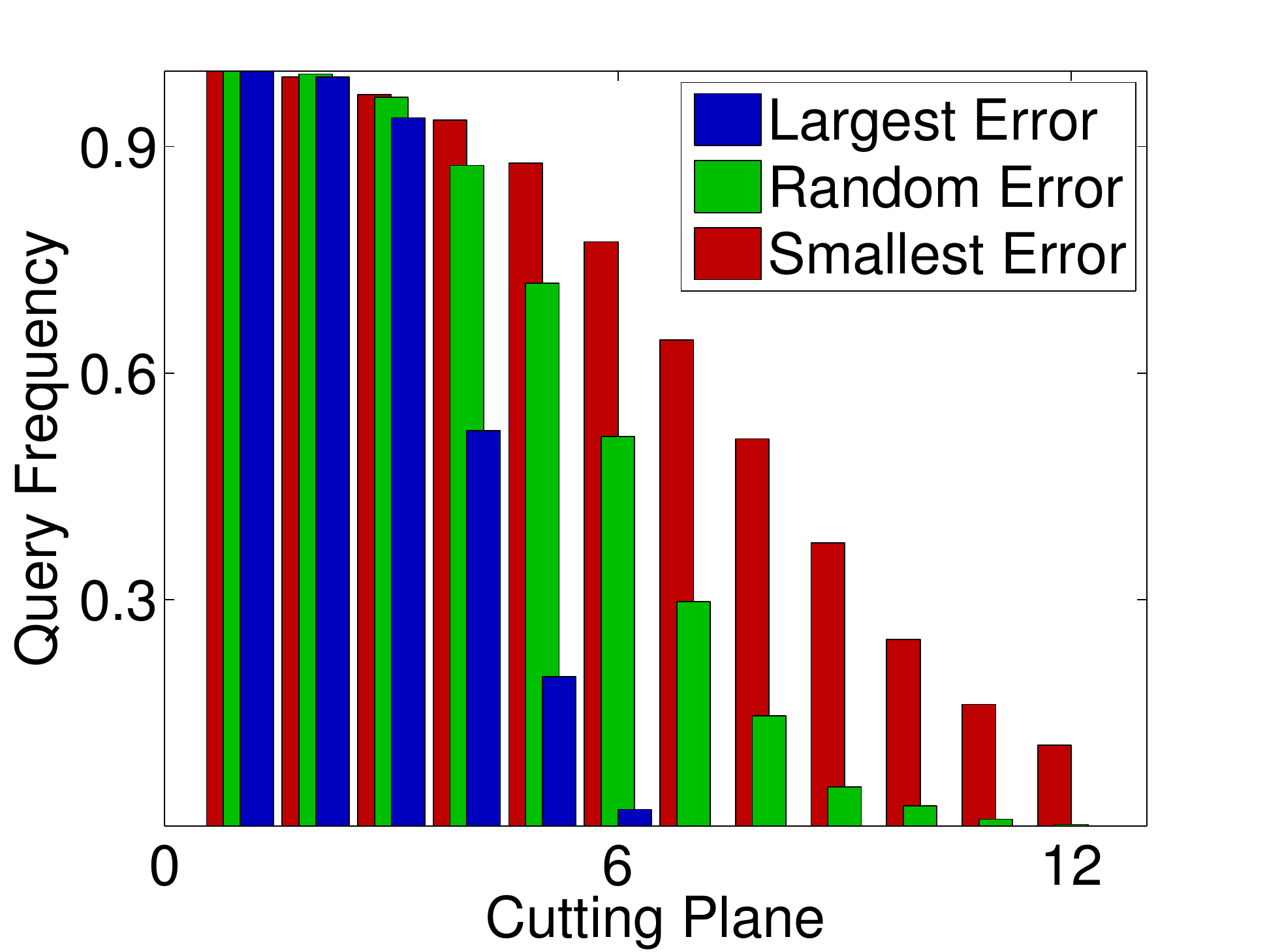}
\end{subfigure}%
\hfill
\begin{subfigure}{0.38\textwidth}
\centering
\includegraphics[width=1\linewidth]{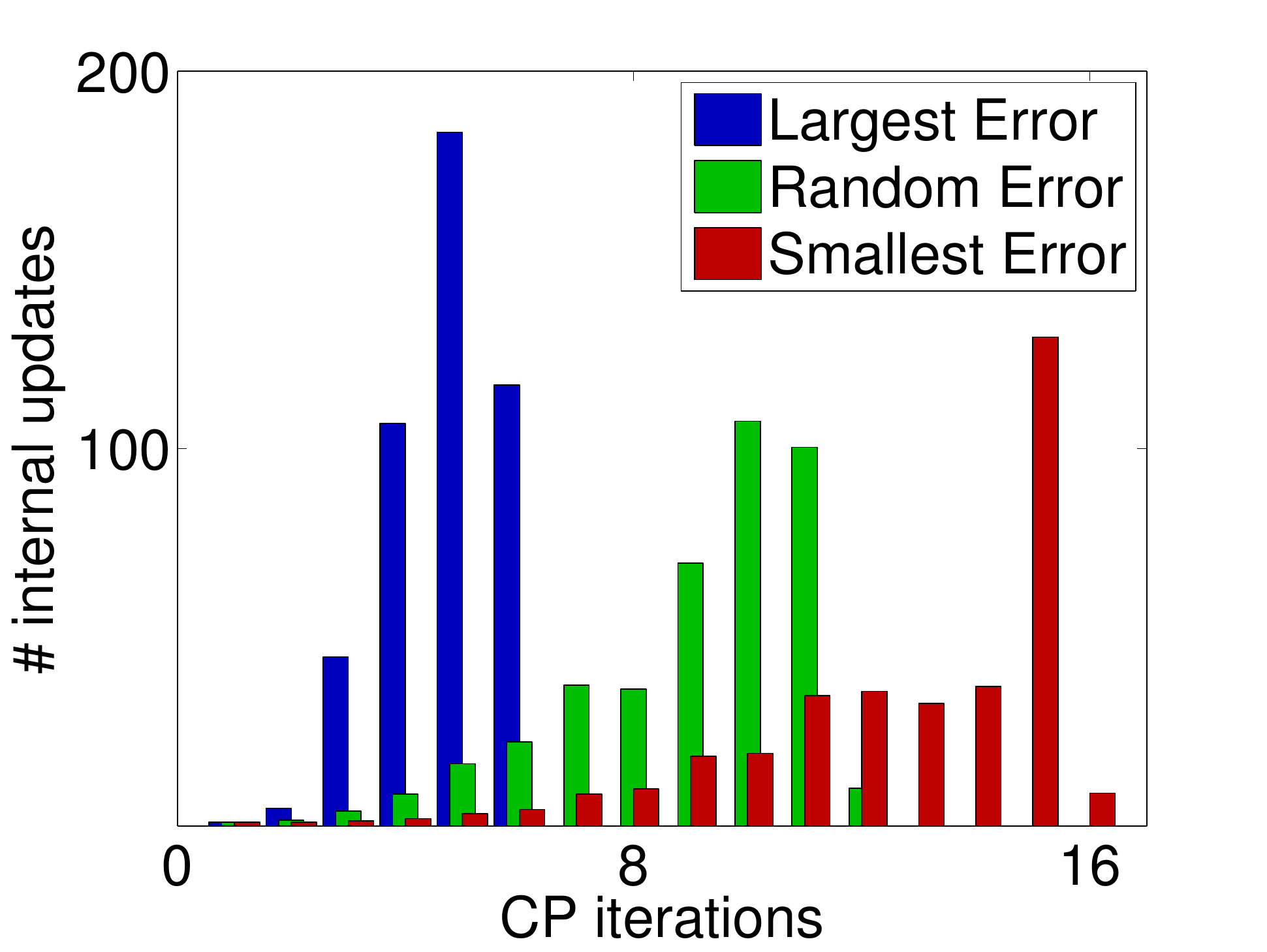}
\end{subfigure}
\caption{\emph{Left} : for each value $i$ the bar represents the empirical probability 
(over $1,000$ runs) to query at least $i$ cutting planes.
\emph{Right}: each bar represents the number of internal Perceptron updates
computed after each Cutting Plane loop.} \label{fig:Perceptron}
\end{figure*}

The left pane of Figure~\ref{fig:Perceptron} supports the soundness of our approach in the case of
a compression scheme with no more than $6$ cutting planes for the best strategy (Largest Error).
Additionally, we can observe a sharp decrease after the third cutting plane
with this strategy and $80\%$ of the time, only $4$ cutting planes are
required to model the dataset. 
In contrast, the right-hand side of Figure~\ref{fig:Perceptron} reveals a trade-off between the number of
cutting planes used and the number of internal updates for each cutting plane. We observe a smooth shift across our three strategies with
Smallest Error putting the emphasis on small number of internal updates.
In all respect, the Random Error strategy acts as a middle ground between
the two other extreme approaches.

For the second experiment the margin (i.e. the volume of $\versionspace$) is
variable with values between $0.01$ and $0.3$. We also monitor the total number
of internal updates rather than the \emph{per cutting plane} value for the
three strategies and a regular Perceptron Algorithm
\footnote{More precisely, we use the exact
same Perceptron than the one used for the internal loop but ran on the full
dataset}.
Remind that this value is bounded from Proposition \ref{prop:ActiveNovikoff}. This bound
also holds for the regular Perceptron.

\begin{figure*}[t]
\begin{subfigure}{0.38\textwidth}
\centering
\includegraphics[width=1\linewidth]{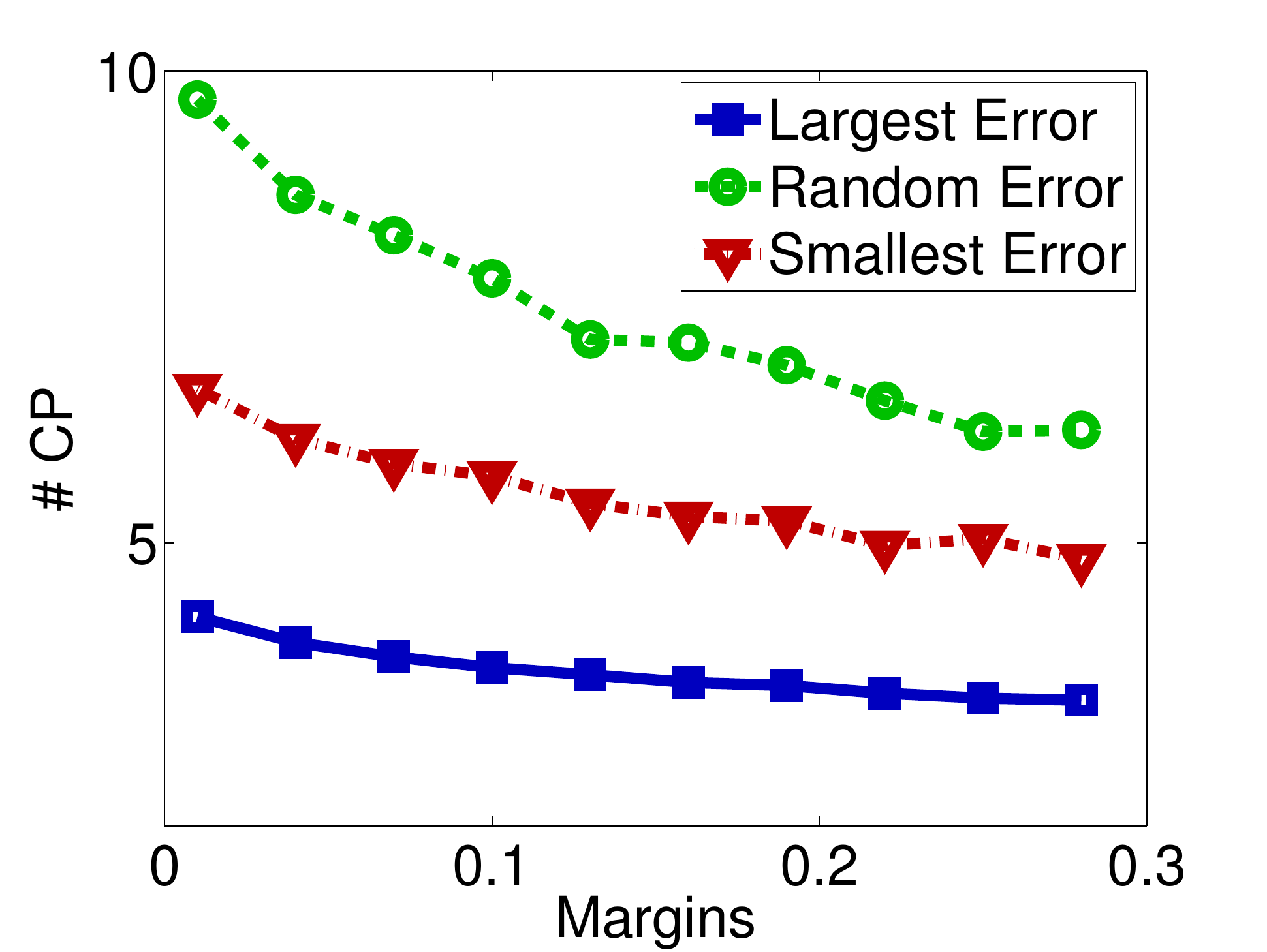}
\end{subfigure}%
\hfill
\begin{subfigure}{0.38\textwidth}
\centering
\includegraphics[width=1\linewidth]{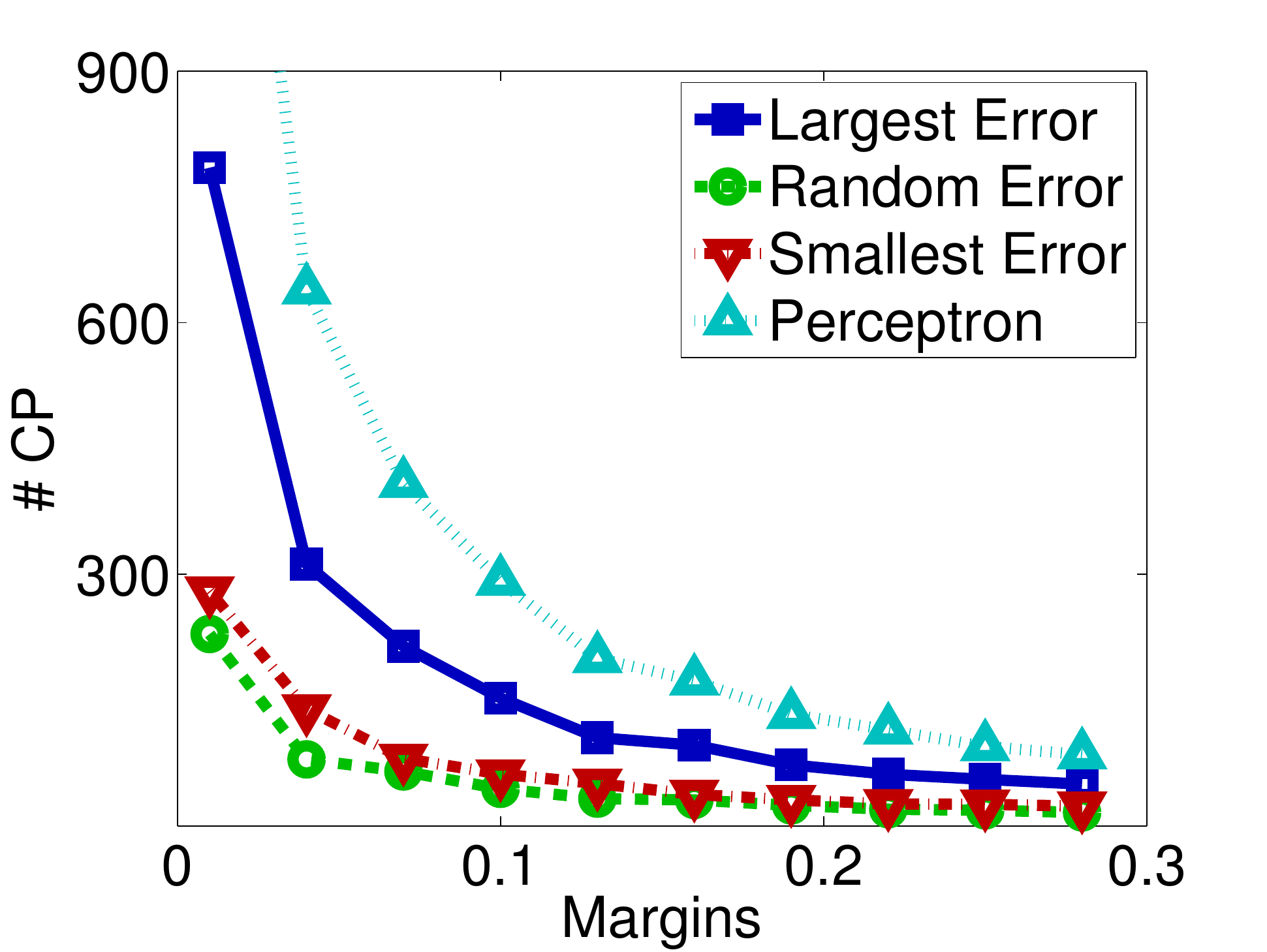}
\end{subfigure}
\caption{\emph{Left}: The average number of cutting planes used for each strategy with
respect to the value of $\gamma$. 
\emph{Right}: the total number of internal updates with respect to $\gamma$.
The fourth plot corresponds to a regular Perceptron 
} \label{fig:PerceptronMargin}
\end{figure*}

The previously observed behavioral shift across the three strategies is
confirmed by Figure~\ref{fig:PerceptronMargin}. Additionally, some relative
robustness is observed with respect to $\gamma$, especially when \changeok{the
number of cutting planes is low}{the emphasis is put on querying a small number
of cutting planes}.
It is interesting to note that the Random Strategy makes nearly as few updates
as Smallest error while still querying a---relatively---low number of cutting planes. Finally, all three strategies are making slightly less updates
than the regular Perceptron. To conclude, note that the theoretical bound of
Proposition~\ref{prop:ActiveNovikoff} is far too big to be plotted on
the plot on the left of Figure~\ref{fig:Perceptron}.

\subsection{Active Learning on Real Data}

We illustrate our method for active learning on text classification data.
For easy comparison, we follow an experimental procedure similar as the one in
\cite{Tong02}. Namely, we use the
\emph{Reuters-21578} ---\emph{ModApte} variation--- and
Newsgroups datasets\footnote{Available at http://www.cad.zju.edu.cn/
home/dengcai/Data/TextData.html}. The Reuters
dataset is composed of $8,293$ documents represented in TF-IDF form for
$18,933$ words. The dataset spans $65$ topics such as \emph{Earn},
\emph{Coffee} or \emph{Cocoa} and is split in $5,946$ training examples and
$2,347$ test examples. On the other hand, the Newsgroups dataset
accounts for $18,846$ documents of $26,214$ features splitted in $20$ topics.
Half of this dataset is uniformly picked for training while the rest is kept for
testing purposes. On
both datasets we train a ``one-versus-all'' classifier for each class.
We start by creating a pool of unlabeled training examples sampled from the
training set. Then we run Algorithm \ref{tab:ActiveCP}.
We use two variations of the {\sc Query()} function: one based on the
Chebyshev center (note that this is equivalent to the \activesvm of
\cite{Tong02}), and the other based on an approximation of the center of gravity from Minka's
Expectation Propagation method \cite{Minka13}.
This last approach corresponds to the \activebpm algorithm and has, to the best
of our knowledge, never been used before. It is a direct application of Active
Learning algorithms with Cutting planes method to the Bayes Point Machine. \changeok{However, keep in mind that
for computational reason we use a rough approximation of the Center of Gravity.}{}
For both methods, we use two pools of different sizes ($500$ and $1,000$ examples).
For initialization reasons, each pool comes with two already labeled vectors.\footnote{SVM and CC are computed with libSVM:
http://www.csie.ntu.edu.tw/~cjlin/libsvm/. BPM and CG are computed from Minka's
own implementation of EP for BPM in matlab:
http://research.microsoft.com/en-us/um/people/minka/papers/ep/bpm/}
All the computations are done with a linear kernel and the presented results are class-wise accuracy
measurements on the test examples over the $10$ most represented classes. The
values reported here are an average of these measures over $25$ runs.
We complement these two datasets with Gunnar Raetsch's
Banana dataset. The Banana dataset is a widely used bataset of $2$-dimensionnal
points split into two classes from which we extract $400$ training and $4900$
test examples.
Due to its small size, the whole training set is used for the pool of
unlabeled example. The computations are realized with an RBF kernel of parameter
$\sigma = 0.5$ and presented results are averaged over $50$ runs.

\begin{figure*}
\centering
\begin{subfigure}{0.33\textwidth}
\centering
\includegraphics[width=1\linewidth]{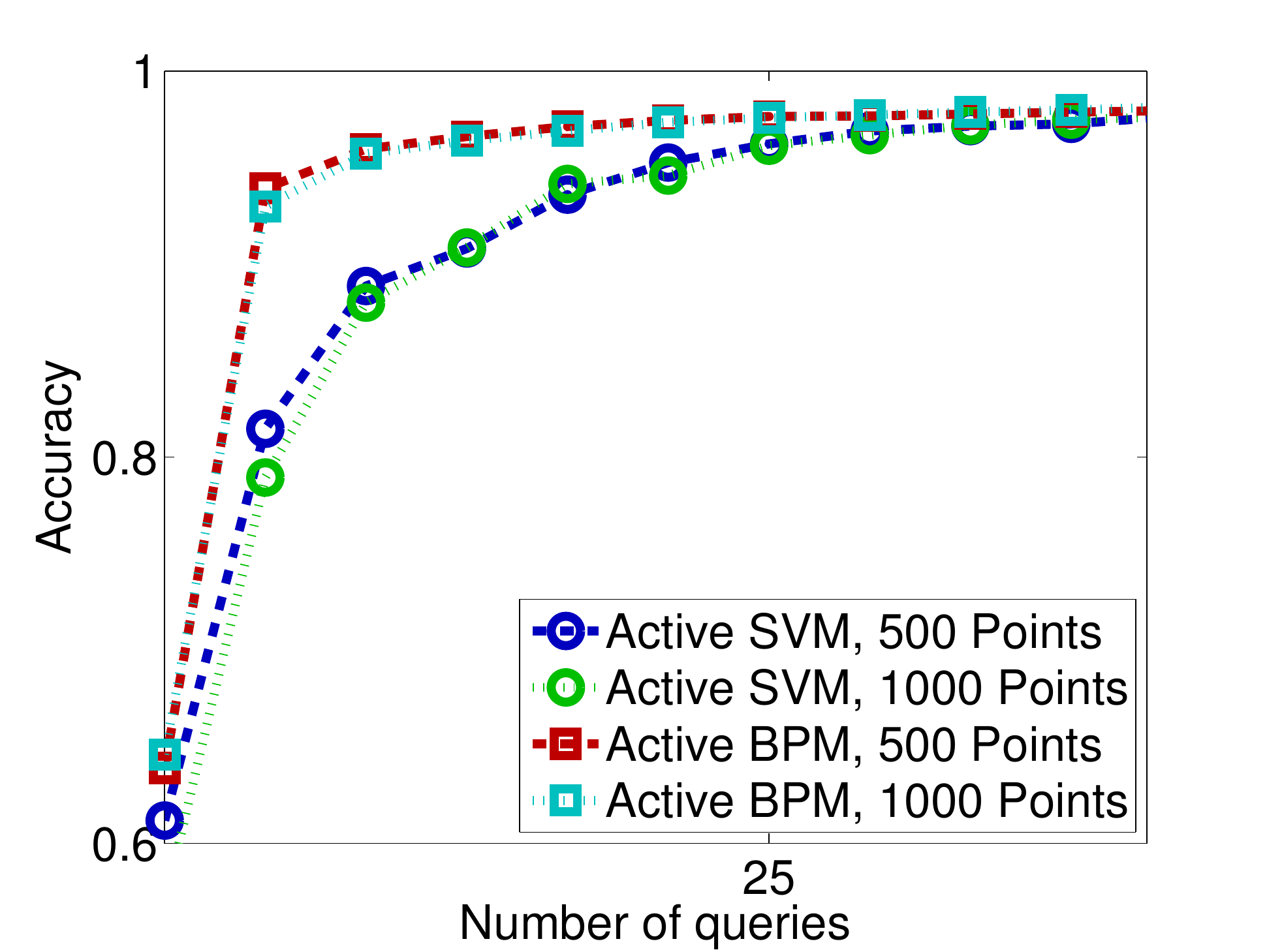}
\end{subfigure}%
\hfill
\begin{subfigure}{0.33\textwidth}
\centering
\includegraphics[width=1\linewidth]{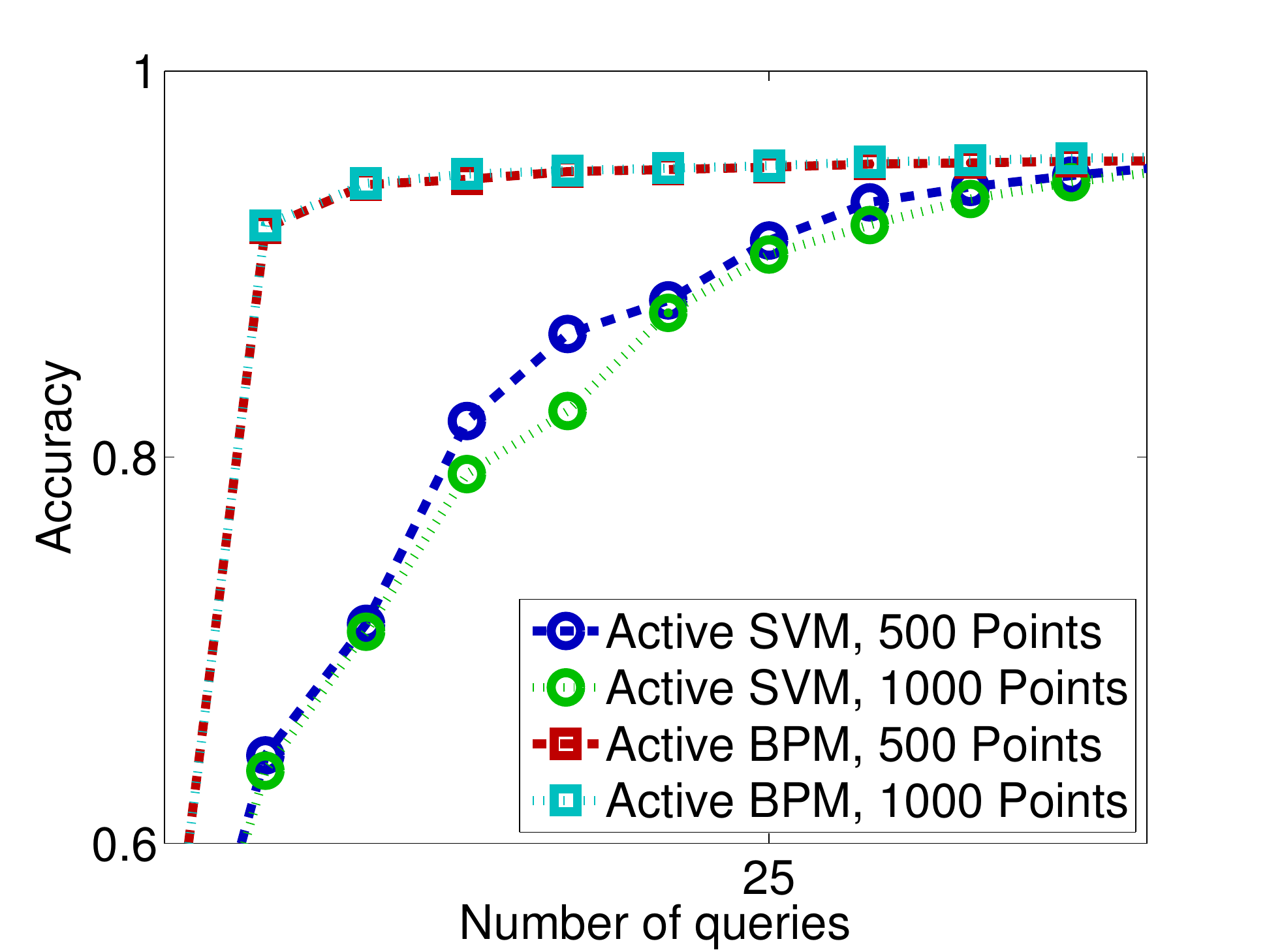}
\end{subfigure}%
\begin{subfigure}{0.33\textwidth}
\centering
\includegraphics[width=1\linewidth]{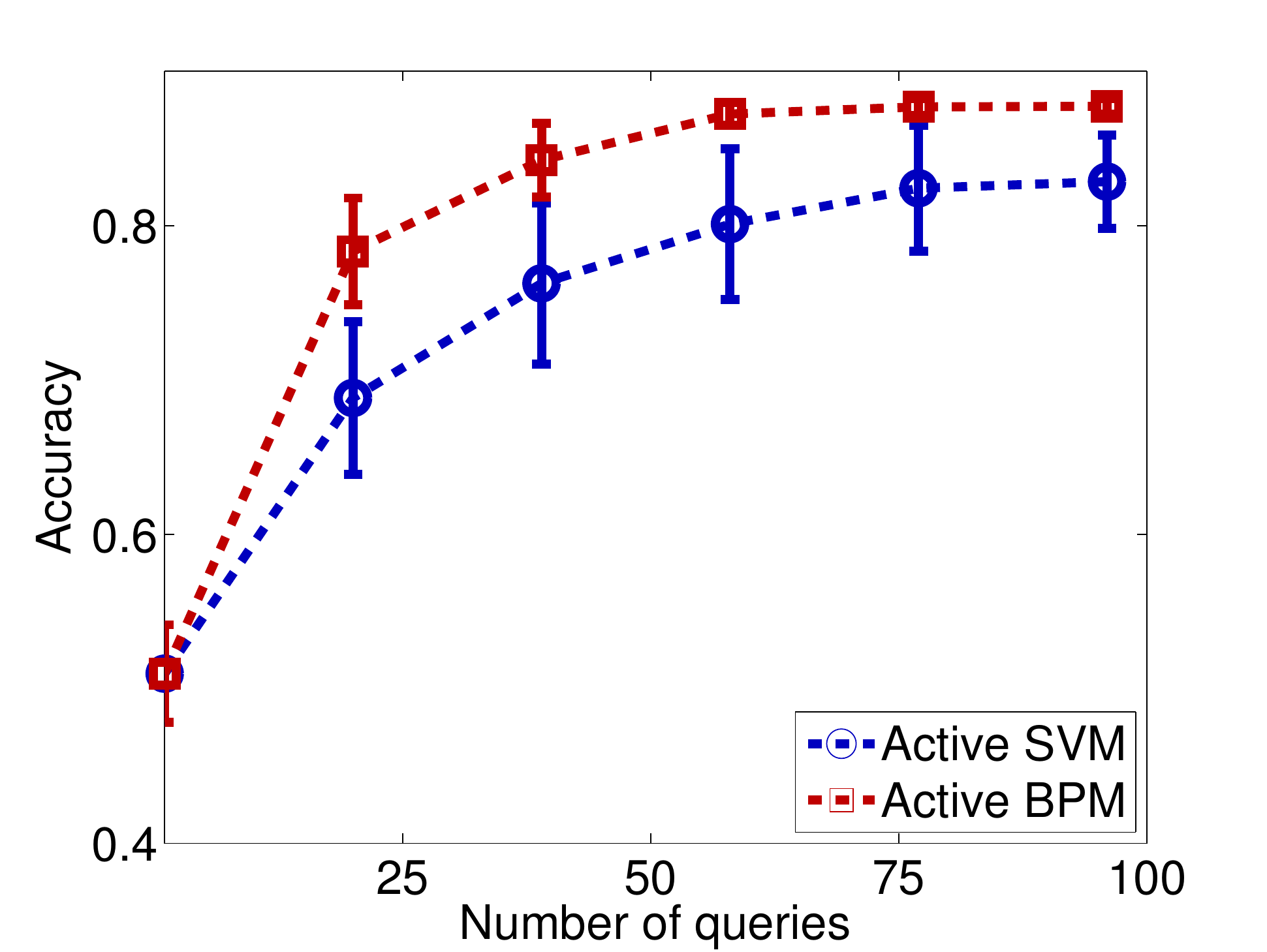}
\end{subfigure}%
\caption{Accuracy  on the Reuters (left) and Newsgroups (middle) datasets
for \activesvm and \activebpm for pools of $500$ and $1000$ examples. Left:
accuracy  with error bars on the Banana dataset (Gunnar Raetsch) for \activesvm and \activebpm.}
\label{fig:result}
\end{figure*}

Figure \ref{fig:result} graphically depicts the behavior of the so-called
\activesvm \cite{Tong02} and the \activebpm algorithms on each
dataset.
Namely, in both algorithms, the queries are selected according to their distance to the
``centroid'' of $\convex$, which, in
turn, serves as classifier.
The difference between these two algorithms lies in that \activesvm uses the
Chebyshev center and \activebpm the center of gravity for centroid.
In Figure~\ref{fig:result}, data are represented by
circles of squares whether they correspond to results achieved by \activesvm or \activebpm.
Additionally, for the Reuters and Newgroups datasets, dashed plots
correspond to the pool of $500$ examples while dotted plots relate to the pool
of $1000$ examples. The error bounds on the third plot (Banana)
correspond to the usual standard deviation. Each
plot represents the accuracy of those algorithms with respect to the number of
queries made. We can see that \activebpm systematically outperforms \activesvm and increases
 its accuracy faster for all datasets, already
attaining an accuracy of $0.9$ after roughly $10$ queries for both
Reuters and Newsgroups datasets.
Both algorithm seem to stabilize after $30$ queries, with the \activebpm being
slightly more accurate than its SVM counterpart.
For the Banana dataset, the accuracy increase in the first queries
is a lot smoother, with an accuracy for \activebpm of roughly $0.8$ after
$20$ queries. Both algorithms seem to have converged after $60$ queries.
Comparatively, not only does \activebpm clearly dominate its SVM counterpart but it
is also more stable as evidenced by the error bars which become negligible
past the $60^{\text{th}}$ query. \changeok{Figure \ref{fig:Dist} shows the
distance of the $i^{\text{th}}$ query to the relevant centroid for each
algorithm. Both algorithms query the closest point to the relevant centroid,
and from Theorem \ref{th:magic} the distance of the query should be related to
the increase in accuracy (at least for the Active BPM). From the plots, it is
clear that the Active BPM algorithm behave as expected, with the distance
increasing as the accuracy stabilizes. Also, using the $1000$ examples
intuitively leads to shorter distances. Although, this does not seems to be the
case for Active SVM and as far as Figure \ref{fig:Dist} is concerned there is
no clear relation between the distance of a query and the corresponding
increase in accuracy.}{}

\section{Conclusion and Future Directions}
\label{sec:conclusion}
In this paper, we have shown that deep connections exist between
Localization methods and Learning algorithms. Both fields have extensively
characterized and studied similar concepts over the past years, sometime
independently. On the other hand, complementary results have been found in each
community. A notable example is the absence of a kernel approach in the Cutting
Planes literature while center of gravity methods were mostly unknown in machine
learning until Herbrich's BPM \cite{Herbrich01}. We may also
mention that the Cutting planes' equivalent of the famous SVM \cite{Vapnik95}
appears as soon as the $70$'s in \cite{Elzinga75}.
This work is a testimony on how it is possible to derive new learning
algorithms, both efficient and theoretically funded, by reformulating Cutting
Planes approach for the learning paradigm.
Besides the cutting plane-related flavor of the present work, it should be
restated that Theorem~\ref{th:generalized} has a value that 
goes beyond the scope of this paper. A field that may be impacted by this
result is obviously that of computational geometry where most of the results
about the computation of centers of gravity come from; 
nonetheless,
it should be noted that more closely related works could also benefit from our result.
For instance, if we consider the
active learning methods whose query steps rely
on explicit exploration of all the possible query/label combinations (see, e.g.
\cite{Roy01}), then Theorem~\ref{th:generalized} provides a tool to devise natural and
theoretically sound heuristics to effectively locate the most informative query
points, or, in other words, those that may lead to the smallest expected error.

 Among all the possible extensions of
this work, one we are particularly interested in is to study how these results
may carry over to the multiclass setting and provide proper multiclass active
algorithms based on, for example, Crammer's Ultraconservative Additive
Algorithms \cite{Crammer2003}.

\begin{small}
\bibliographystyle{./IEEEtran}
\bibliography{activecp}
\end{small}


\appendix
This appendix is composed of three sections. Section \ref{sec:prelim} serves as
reminder of basic notions and results for the proofs of the other sections,
additionally, we will introduce our set of notation thorough this section.
Section \ref{sec:grunbaum} consists in a rewriting of the proof of Grunbaum in
\cite{grunbaum60} on the partition of convex bodies by hyperplanes. The proof is
restated in full with proper notation as it is the starting point of our result.
The last section gives the proof of theorem \ref{th:main} which is an extended
version of the result of Grunbaum and is one of the contribution of our paper.

\section{Preliminaries} \label{sec:prelim}

\subsection{Hyper-Sphere and Hyper-Ball}

\begin{definition}[$n$-dimensional Sphere]
We call $n$-sphere of center $O \in \realset^n$ and radius $R \in
\realset$ and write $\Sphere(O, R) \subset \realset^n$ the subset
\[
	\Sphere(O, R) \doteq \lbrace x \in \realset^n : \norm{x - O} = R \rbrace
\]
\end{definition}

\begin{definition}[$n$-dimensional Ball] \label{def:Ball}
We call $n$-ball of center $O \in
\realset^n$ and radius $R \in \realset$ and write $\Ball(O,R) \subset
\realset^n$ the subset
\[
	\Ball(O,R) \doteq \lbrace x \in \realset^n : \norm{x - O} < R \rbrace
\]

Alternatively, one can think of a ball as :
\[
	\Ball(O,R) \doteq \bigcup_{r \in [0, R]} \Sphere(0, r)
\]
\end{definition}

\begin{definition}[Surface of a spehe]
We call \emph{Surface} of the $n$-sphere $\Sphere(O, R)$ and write $\vol(
\Sphere(O, R) )$ the $n-1$ dimensional volume
\[
\vol( \Sphere(O, R) ) \doteq \Cpis_n \times R^{n-1}
\]
Where $\Cpis_n$ is a constant factor depending only on $n$ (e.g $\Cpis_1 = 2$,
$\Cpis_2 = 2 \pi$, $\Cpis_3 = 4 \pi$ and so on \ldots)
\end{definition}

\begin{definition}[Volume of a Ball]
We call \emph{Volume} of the $n$-ball $\Ball(O,R)$ and write $\vol( \Ball(O, R)
)$ the $n$-dimensional volume
\[
\vol( \Ball(O, R) ) \doteq \int_0^R \vol( \Sphere( O, r ) ) dr
\]
That is
\begin{align*}
\vol( \Ball(O, R) ) & = \int_0^R \Cpis_n R^{n-1} dr \\
& = \frac{\Cpis_n R^{n}}{n} \\
& =	\Cpi_n R^{n}
\end{align*}

Where $\Cpi_n \doteq \frac{\Cpis_n}{n}$ is a constant factor depending only on
$n$.
\end{definition}

\subsection{Hyper-Cone}
From these core definitions, we can now introduce (Hyper)-cones and some of
their core properties.
Intuitively, an Hyper-cone of dimension $n+1$, center $O$, radius $R$ and height
$H$ is a sequence of $n$-Ball of linearly decreasing radius between $R$ and $0$,
each one living on a difference ``slice'' of $\realset^{n+1}$ between $O$ and
$0+H$.

\begin{remark}
We will use $\unit_{n+1}$ to denote the vector of $\realset^{n+1}$ with $1$ on
its $n+1$ component and $0$ elsewhere.
\end{remark}

\begin{definition}[Hyper-cone]
We call Hyper-cone of dimension $n+1$, base $\Ball(O, R) \subset \realset^n$ and
height $H$ the set : 
\[
H \doteq \bigcup_{\forall h \in [0,H]} \Ball \left( O + h\unit_{n+1},
\frac{H - h}{H} R \right)
\]
\end{definition}

Alternatively, we can define the \emph{apex} $Z \doteq O + H
\times \unit_{n+1}$ of the hyper-cone and give the following definition :

\begin{definition}[Hyper-cone (2)]
We call Hyper-cone of dimension $n+1$, base $\Ball(O, R) \subset \realset^n$ and
apex $Z$ the convex hull $\conv \left( \lbrace \Ball(O, R);  z \rbrace \right)$.
\end{definition}

We are now ready to state the core properties of Hyper-cone that we will use in
the remaining of this document.

We start with the volume of a Hyper-cone

\begin{definition}[Volume of Hyper-cone]
Given an Hyper-cone $C \in \realset^{n+1}$ of dimension $n+1$, base $\Ball(O,
R) \subset \realset^n$ and height $H$ we call \emph{volume} and write $\vol(C)$
the $n+1$-dimensional volume :
\[
\vol(C) \doteq \int_0^H \vol \left( \Ball \left( O + h \unit_{n+1},
\frac{H - h}{H}R \right) \right) dh
\]
\end{definition}

\begin{proposition} \label{prop:VolCone}
The volume of the Hyper-cone $C \subset \realset^{n+1}$ of dimension $n+1$, base
$\Ball(O, R) \subset \realset^{n}$ and height $H$ is 
\[
\vol(C) = \frac{\Cpi_n R^{n}}{n+1}H
\]
\end{proposition}
\begin{proof}
From the definition of volume of a sphere we have $\vol( \Ball(O, R) ) \doteq
\Cpi_n R^{n}$. Substituting $R$ by $\frac{H - h} R$ and from the definition
of the volume of a Hyper-cone we have

\[
\vol(C) = \int_0^H \Cpi_n \left( \frac{H - h}{H}R \right) ^{n} dh
\]

We substitute $h$ by $u \doteq H - h$, $du = -dh$.

\begin{align*}
\vol(C)  & = \int_H^0 - \Cpi_n \left( \frac{u}{H}R \right) ^{n} du \\
& = \int_0^H \Cpi_n \left( \frac{u}{H}R \right) ^{n} du \\
& = \frac{\Cpi_n R^{n}}{H^{n}} \int_0^H u^{n} du \\
& = \frac{\Cpi_n R^{n}}{H^{n}} \times \frac{H^{n+1}}{n+1} \\
& = \frac{\Cpi_n R^{n}}{n+1}H
\end{align*}
\end{proof}

\subsection{Center of Mass}

With the previous definition formally stated, we can now go one step further and
define the \emph{center of mass} or sometimes called \emph{center of gravity}

\begin{definition}[Center of gravity]
For a given convex body $X \subset \realset^n$ we call \emph{center of gravity}
\footnote{For a more complete definition, we should take into account the mass
distribution over $X$. Although, in an effort to keep
things simple, we assume a uniformly distributed mass.} 
and
write $\cg(X) \in \realset^n$ the point :
\[
\cg(X) = \frac{1}{\vol(X)}\int_{x \in X} x dx
\]
\end{definition}

\begin{proposition} \label{prop:cgS}
Let $S$ a $n$-dimensional sphere such that $S \doteq \Sphere(O, R)$. Then
$\cg(S) = O$.
\end{proposition}
\begin{proof}
Without loss of generality, assume that $O = 0$. Then, $S = \lbrace x \in
\realset^n : \norm{x} = R \rbrace$. Since $\norm{x} = \norm{-x}$ it is clear
that $\forall x \in \realset^n : x \in S \Leftrightarrow -x \in S$.
Thus, we can rewrite $\cg(S)$ as 
\[
\cg(S) = \frac{1}{\vol(S)} \int_{x \in S} -x dx
\]

Thus
\begin{align*}
2 \cg(S) & = \frac{1}{\vol(S)} \left( \int_{x \in S} x dx + \int_{x \in S} -x dx    
\right) \\
& = \frac{1}{\vol(S)} \int_{x \in S} x - x dx \\
& = 0
\end{align*}
\end{proof}

\begin{proposition}\label{prop:cgB}
Let $B$ a $n$-dimensional ball such that $B \doteq \Ball(O, R)$. Then $\cg(B) =
O$.
\end{proposition}
\begin{proof}
Remind that $B$ can be seen as a collection of concentric $n$-sphere of center
$O$ and radii between $0$ and $R$ (see Def. \ref{def:Ball}). Then, we can
rewrite $\cg(B)$ as
\begin{align*}
\cg(B) & = \frac{1}{\vol(B)} \int_0^R \cg( \Sphere(O, r) ) \vol( \Sphere(O, r) )
dr
\\
& = O
\end{align*}
Where the last line come from Proposition \ref{prop:cgS}.
\end{proof}

\begin{proposition}[Center of Gravity of an Hyper-cone] \label{prop:cgHc}
Let $C$ a $n+1$ dimensional Hyper-cone ($C \subset \realset^{n+1}$) of base
$\Ball(O, R) \subset \realset^n$ and apex $Z$ such that $\norm{Z - O} = H$.
Then, $\cg(C)$ is located on the segment $[O; Z]$ at a distance $H/_{n+2}$ of
$O$.
\end{proposition}
\begin{proof}
Without loss of generality, we assume that $O = 0$ and $Z = H \unit_{n+1}$. By
definition, $C$ is a collection of ball, and we can
rewrite $\cg(C)$ as :
\[
\cg(C) = \frac{1}{\vol(C)} \int_0^H \cg \left[ \Ball \left( h \unit_{n+1},
\frac{H - h}{H} R \right) \right] \times \vol \left[ \Ball \left( h \unit_{n+1},
\frac{H - h}{H} R \right) \right] dh
\]

From Proposition \ref{prop:cgB} it is clear that $\cg(C)$ lies on the
segment $[O; Z]$. The remaining of the proof came by explicitly calculating
$\cg(C)$.

\begin{align*}
\cg(C) & = \frac{1}{\vol(C)} \int_0^H h \unit_{n+1} \times \vol \left[ \Ball
\left( h \unit_{n+1}, \frac{H - h}{H} R \right) \right] dh \tag{Prop. \ref{prop:cgB}} \\
& = \frac{1}{\vol(C)} \int_0^H h \unit_{n+1} \frac{\Cpi_n R^{n}}{H^{n}}
(H - h)^{n} dh \tag{Volume of a Ball} \\
& = \frac{1}{\vol(C)} \int_0^H (H - u) \unit_{n+1} \frac{\Cpi_n R^{n}}{H^{n}}
u^{n} du \tag{Subst. $u \doteq H - h$, see Prop. \ref{prop:VolCone}} \\
& = \frac{1}{\vol(C)} \left[ \int_0^H \frac{\Cpi_n R^{n} H
\unit_{n+1}}{H^{n}} u^{n} du - \int_0^H \frac{\Cpi_n R^{n} \unit_{n+1}}{H^{n}}
u^{n+1} du \right] \\
& = \frac{1}{\vol(C)} \left[ \frac{\Cpi_n R^{n} \unit_{n+1}}{H^{n-1}} \int_0^H
u^{n} du - \frac{\Cpi_n R^{n} \unit_{n+1}}{H^{n}} \int_0^H u^{n+1} du
\right] \\
& = \frac{1}{\vol(C)} \left[ \frac{\Cpi_n R^{n} H^{n+1} \unit_{n+1} }{H^{n-1}
(n+1)} - \frac{\Cpi_n R^{n} H^{n+2} \unit_{n+1}}{H^{n} (n+2)} \right] \\
& = \frac{1}{\vol(C)} \left[ \frac{\Cpi_n R^{n} H^2 \unit_{n+1}}{n+1} -
\frac{\Cpi_n R^{n} H^2 \unit_{n+1}}{n+2} \right] \\
& = \frac{n+1}{\Cpi_n R^{n} H} \left[ \frac{\Cpi_n R^{n} H^2
\unit_{n+1}}{n+1} -
\frac{\Cpi_n R^{n} H^2 \unit_{n+1}}{n+2} \right] \tag{Volume of a Hyper-cone}
\\
& = \left( H - H\frac{n+1}{n+2} \right) \unit_{n+1} \\
& = \left( 1 - \frac{n+1}{n+2} \right) H \unit_{n+1} \\
& = \left( \frac{n+2-n-1}{n+2} H \unit_{n+1} \right) \\
& = \left( \frac{H}{n+2} \right) \unit_{n+1}
\end{align*}

That is to say, $\cg(C)$ is on the segment $[O; Z]$ at a distance $H/_{n+2}$ of
$O$.

\end{proof}

\subsection{Hyperplane and Halfspace}

\begin{definition}
We call \emph{($n$)-Hyperplane} of normal $w \in \realset^n$ and offset $b \in
\realset$ and write $\Hplane (w, b) \subset \realset^n$ the subset : 
\[
\Hplane( w, b) \doteq \lbrace x \in \realset^n : \dotProd{w}{x} + b = 0 \rbrace
\]
\end{definition}

\begin{definition}
We call \emph{Positive Halfspace} of the $n$-Hyperplane $\Hplane (w, b) \subset
\realset^n$ and write $\Hp (w, b) \subset \realset^n$ the subset
\[
\Hp (w, b) \doteq \lbrace x \in \realset^n : \dotProd{w}{x} + b \geq 0 \rbrace
\]

Conversely, we call \emph{Negative Halfspace} of $\Hplane (w, b) \subset
\realset^n$ and write $\Hm (w, b) \subset \realset^n$ the subset
\[
\Hm (w, b) \doteq \lbrace x \in \realset^n : \dotProd{w}{x} + b \leq 0 \rbrace
\]
\end{definition}

Additionally, note that $\Hplane (w, b) \subset \Hp (w, b)$ but $\Hplane (w, b)
\not\subset \Hm (w, b)$.

\begin{definition}
For any subset $X \subset \realset^n$ and any Hyperplane $\Hplane \subset
\realset^n$ we call \emph{Positive Partition} and write $X^+ \subset \realset^n$
the subset 
\[
X^+ \doteq X \cap \Hp
\]

Conversely, for we call \emph{Negative Partition} and write $X^- \subset
\realset^n$ the subset
\[
X^- \doteq X \cap \Hm
\]
\end{definition}

\begin{proposition}[Volume reduction of Hyper-Cone] \label{prop:HconeRatio}
For any $(n+1)$-Hyper-cone of base $\Ball(O, R)$, apex $Z$ and Height $H$,
let set $\Hplane_{\cg(C)} \doteq \Hplane (\unit_{n+1}, H/_{n+2})$ the
Hyperplane passing by $\cg(C)$ ( i.e. $\cg(C) \in \Hplane_{\cg(C)}$ )
and parallel to $\realset^n$. Then,
\[
\vol(C^+) = \vol(C) \left[ \frac{1}{ \left( 1 + \frac{1}{n+1} \right)^{n+1}}
\right] \geq \vol(C) e^{-1}
\]
\end{proposition}

\begin{proof}
We start by proving the right-hand side of the relation. Let set $n' = n+1$ and
divide both side by $\vol(C)$ then we can rewrite it as
\[
\frac{1}{ \left( 1 + \frac{1}{n'} \right)^{n'}} \geq e^{-1}
\]

From the usual definition of $e$ we have that
\begin{align*}
& \lim_{n \rightarrow \infty} \left( 1 + \frac{1}{n'} \right)^{n'} = e \\
\Leftrightarrow & \lim_{n' \rightarrow \infty} \frac{1}{ \left( 1 + \frac{1}{n'}
\right)^{n'}} = e^{-1}
\end{align*}

And by standard arguments we can show that
\[
\frac{1}{ \left( 1 + \frac{1}{n'}
\right)^{n'}} \geq \frac{1}{ \left( 1 + \frac{1}{n'+1}
\right)^{n'+1}}
\]

Therefore
\[
\frac{1}{ \left( 1 + \frac{1}{n'}
\right)^{n'}} \geq \lim_{n \rightarrow \infty} \frac{1}{ \left( 1 + \frac{1}{n'}
\right)^{n'}} = e^{-1}
\]

Finally, the left-hand side of relation is obtained by direct calculation of
$\vol(C+)$. The general idea is the same than the calculation of $\vol(C)$ but,
instead of integrating over the entire height we start at $H/_{n+2}$, thus
ignoring $C^-$. Besides, without loss of generality, we assume that $O = 0$ and
that $Z = H\unit_{n+1}$.

\begin{align*}
\vol(C^+) & = \int_{H/_{n+2}}^{H} \vol \left[
\Ball \left( h \unit_{n+1} , \frac{H - h}{H} R \right) \right] dh \\
& = \int_{H/_{n+2}}^{H} \Cpi_n \frac{R^{n}}{H^{n}} (H - h)^{n} dh \tag{Volume of
a Ball}\\
& = \int_{0}^{H \left( 1 - \frac{1}{n+2} \right)} \Cpi_n
\frac{R^{n}}{H^{n}} u^{n} du \tag{Subst. $u \doteq H - h$}\\
& = \frac{\Cpi_n R^{n}}{H^{n}} \int_0^{H \left( 1 - \frac{1}{n+2} \right)}
u^{n} du \\
& = \frac{\Cpi_n R^{n}}{H^{n}} \times \frac{H^{n+1}}{n+1} \times \left(
1 - \frac{1}{n+2} \right)^{n+1} \\
& = \frac{\Cpi_n R^{n} H}{n+1} \times \left( 1 - \frac{1}{n+2} \right)^{n+1}
\\
& = \vol(C)\left( 1 - \frac{1}{n+2} \right)^{n+1} \\
& = \vol(C)\left(\frac{n + 1}{n+2} \right)^{n+1} \\
& = \vol(C)\left(\frac{(n + 1) \times \frac{1}{n+1}}{(n+2) \times \frac{1}{n+1}}
\right)^{n+1} \\
& = \vol(C)\left(\frac{1}{\frac{n+2}{n+1}} \right)^{n+1} \\
& = \vol(C)\left(\frac{1}{1 + \frac{1}{n+1}} \right)^{n+1} \\
& = \vol(C) \left[ \frac{1}{ \left( 1 + \frac{1}{n+1} \right)^{n+1} } \right]
\end{align*}

\end{proof}

\subsection{Setting}

For the remaining of this document, let $\K$ be a (full dimensional)
convex body in $\realset^{\d}$.

\begin{definition}
For any convex body $K \in \realset^{\d}$ we say that $K$ is \emph{Spherically
Symmetric along the unit vector $\unit$} if and only if $\forall \lambda
\in \realset$ the cut of $K$ by the hyperplane $\Hplane(\unit, \lambda)$ (i.e.
$K \cap \Hplane(\unit, \lambda)$) is a $n$ dimensional hypersphere of center
$\lambda \unit$
\end{definition}

\section{Partition of Convex Bodies By Hyper-Plane} \label{sec:grunbaum}

This section consists in a rewriting of the proof of \cite{grunbaum60}
instantiated within the previously defined notation and setting.

\subsection{Theorem}

\begin{theorem} \label{th:magic2}
For any convex body $\K \subset
\realset^{\d}$, and any hyperplane $\Hplane$.
If $\G(\K) \in \K^+$ then 
\[ \vol(\K^+) \geq e^{-1} \times \vol(\K) \]
\end{theorem}

\begin{proof}

\begin{note}[Points along $\unit_{n+1}$] \label{note:realPoints}
This proof will revolve around key points located on the $n+1^{\text{th}}$ axis
of $\realset^{n+1}$ of base vector $\unit_{n+1}$. In an attempt to avoid
overburdening the notation, we will treat these points as number along the real
line when context is clear. Therefore, if $x = \lambda_1 \unit_{n+1}$ and $y =
\lambda_2 \unit_{n+1}$ we will freely write $x > y$ if $\lambda_1 > \lambda_2$.
\end{note}

Let $\Hplane$ the hyperplane such that $\Hplane = \arg\min_{\Hplane}
\vol(K^+)$ such that $\G(\K) \in \K^+$.
It is easy to see that $\G(\K) \in \Hplane$ : if $\G(\K) \notin \Hplane$ you can
always reduce $\vol(\K^+)$ by shifting $\Hplane$ toward $\G(\K)$.
Without loss of generality, let's say that $\G(\K) = 0$ the origin of
$\realset^{\d}$ and that $\unit_{\d}$ is the normal vector of $\Hplane$ with $b
= 0$, hence $\Hplane = \Hplane(\unit_{\d}, 0)$.

In order to ease the comprehension of the proof, we make the following
assumption that we will lift later on.

\begin{assumption} \label{ass:Sym}
$\K$ is a convex body which is Spherically Symmetric along $\unit_{n+1}$
\end{assumption}

A direct implication of this is that $\G(\K) = \G(\K \cap \Hplane)$. In other
words, $\G(\K)$ is the center of the $n-1$ dimensional sphere $\K \cap \Hplane$
(see, for example, the argument of Prop. \ref{prop:cgHc} ).

Let $\C^+$ the Hyper-cone of base $\K \cap \Hplane$ and apex $Z$ such that $\C^+
\subset \Hp$ and $\vol(\C^+) = \vol(\K^+)$.

Moreover either :
\begin{itemize}
  \item $\K^+ = \C^+$ and $Z$ is the apex of $\K^+$ 
  \item $Z \notin \K^+$
\end{itemize}

To prove that, remember that each \emph{slice} $\K \cap \Hplane(\unit_{\d},
\lambda)$ of $\K$ along the $\d$ axis is a sphere. We look at the function
$r_{\C^+}()$ (resp. $r_{\K^+}()$) which maps each value of $\lambda \in
\realset_+$ with the radius of the corresponding \emph{slice} of $\C^+$ (resp.
$\K^+$).By construction, we know that $r_{\C^+}(0) = r_{\K^+}(0)$ and, by definition
$r_{\C^+}()$ is a deacreasing linear function.
If $r_{\K^+}()$ has any stricly convex part, then there exists an
arc $[r_{\K^+}(\lambda_1), r_{\K^+}(\lambda_2)]$ which is not in $\K^+$ and
therefore $\K$ is not a convex set. Therefore $r_{K^+}()$ is concave. Then,
because $r_{\C^+}(0) = r_{\K^+}(0)$, for $Z$ to be in $\K^+$ either $\K^+$ is a
Hyper-cone (and $r_{\K^+}$ is linear) or $\vol(\K^+) > \vol(\C^+)$ (that is to
say $\int_0^{\infty} r_{\K^+}(\lambda) d\lambda > \int_0^{\infty} r_{\C^+}(\lambda)
d\lambda$ ) which is in contradiction with the definition of $\C^+$

As a consequence, $\C^+$ is at least as elongated as $\K^+$. In other words, the
mass of $\C^+$ is more spread along the axis of $\unit_{n+1}$, this incurs a
shift of the center of gravity of $\G(\C^+)$ with respect to $\G(\K^+)$.
Therefore $\G(\C^+)$ is on the closed segment $\left[
\G(K^+), Z \right]$.

Thus, by using the notation introduced in Note
\ref{note:realPoints} :
\[
0 = \cg(\K) \leq \cg(\K^+) \leq \cg(\C^+) \leq Z
\]

We now define $\C^-$ by extending $\C^+$ such that $\C \doteq \C^- \cup \C^+$ is
a cone of apex $Z$ and $\vol(\C^-) = \vol(\K^-)$. Therefore, 
\begin{align*}
\vol(\C)  & = \vol(\C^+) + \vol(\C^-) \\
& = \vol(\K^+) + \vol(K^-) \\
& = \vol(\K)
\end{align*}

Once again, we are interested in the relative position of $\cg(\K^-)$ and
$\cg(\C^-)$. We invoke the same arguments than before and claim that, in a
similar way :
\[
\cg(\K^-) \leq \cg(C^-) \leq 0 = \cg(\C)
\]

\begin{remark}
The proof for this is a little more tricky this time though. Part of this is due
to the fact that $C^-$ is not
a Hyper-cone in itself and one must consider $\C$ and $\K$ in their entirety
for the nonconvexity argument.
A possible start is to consider the radius increase along the reverse axis $\overline{\unit_{n+1}} \doteq -\unit_{n+1}$ and
replicate the previous argument with added attention to the slope of
$r_{\K^-}()$ which must be such that $r_{\K}()$ as a whole is still concave.
\end{remark}

Let $\alpha, \beta \in \realset$ such that $\alpha \doteq {\vol( \K^+
)}/_{\vol( \K)}$ and $\beta \doteq {\vol( \K^- )}/_{\vol( \K)}$. Then

\[
\G(\K) = \alpha \G(\K^+) + \beta \G(\K^-)
\]

Or alternatively, by construction of $\C$

\[
\G(\C) = \alpha \G(\C^+) + \beta \G(\C^-)
\]

Combining these with the previous inequalities, we have that

\[
\cg(\K) \leq \cg(\C)
\]

Moreover, we know from Proposition \ref{prop:cgHc} that $\cg(\C)$ is at a
distance $H/_{n+2}$ of its base, where $H$ is the height of $\C$.

Let $\widetilde{\Hplane} \doteq \Hplane(\unit_{n+1}, \tilde{b})$ such that
$\cg(\C) \in \widetilde{\Hplane}$ and write $\widetilde{\C^+}$ the positive
partition of $\C$ by $\widetilde{\Hplane}$, that is $\widetilde{\C^+} \doteq
\widetilde{\Hp} \cap \C$.

From Proposition \ref{prop:HconeRatio} we have that $\vol \left(
\widetilde{\C^+} \right) \geq e^{-1} \vol(\C)$. Moreover, because of $\cg(\C)
\geq \cg(\K)$ we have that $\vol(\C^+) \geq \vol \left( \widetilde{\C^+}
\right)$.

Putting all of this together we get that

\begin{align*}
\vol(\K^+) & = \vol(\C^+) \\
& \geq \vol \left( \widetilde{\C^+} \right) \\
& \geq e^{-1} \vol(\C) \\
& = e^{-1} \vol(\K)
\end{align*}

Finally, all we have left is to deal with Assumption \ref{ass:Sym}. This is
simply tackled by remarking that, by definition, spherical symmetrization
preserve volumes along its axis. Thus, for any $\K$ of any convex shape it
suffices to apply the proof on the spherical symmetrization of $\K$ : $\Ssym( \K
)$ and we have
\[
\vol( \K^+ ) = \vol (\Ssym( \K^+ ) ) \geq e^{-1} \vol( \Ssym( \K ) ) = \vol( \K
)
\]

\end{proof}

\section{Generalized Volume Reduction}

This section is dedicated to the main theorem which is a generalization of
Theorem \ref{th:magic2} to approximate center of mass.

\begin{theorem} \label{th:main}
For any convex body $\K \subset \realset^{n+1}$ and any Hyper-plane $\Hplane$ of
normal $\unit$, splitting $\K$ in $\K^+$ and $\K^-$. Let 
\[
\x \doteq \G(\K) + \lambda \frac{(n+1) \vol(\K)}{\Cpi_n R_{\K^+}^n}
\left[ \frac{H_{\K^+}}{(n+2) H_{\K^-}}
\right]^n \left[ 1 - \frac{1}{n+2} \right] \unit
\]

Where $H_{\K^+} = \max_{a \in \K^+} a^T \unit$, $H_{\K_-} = \min_{a \in \K^-}
a^T \unit$ and $R_{\K^+}$ the radius of the $n-1$-Ball $\Ball_{\K \cap \Hplane}$
such that $\vol(\Ball_{\K \cap \Hplane}) = \vol( \K \cap \Hplane )$.

Then, if $\x \in \K^+$ the following holds true
\[
\vol(\K^+) \geq \vol(\K)(1-\lambda)^{n+1}e^{-1}
\]
\end{theorem}
\begin{proof}

The proof start in a similar way than the one of Grunbaum, with respect to $\x$.

Namely, let Assumption \ref{ass:Sym} hold for now, and let $\Hplane$ the
hyperplane such that $\Hplane = \arg\min_{\Hplane} \vol(K^+)$ such that $\x \in \K^+$.
Same as before, we have that $\x \in \Hplane$.
Without loss of generality, let's say that $\x = 0$ the origin of
$\realset^{\d}$ and that $\unit_{\d}$ is the normal vector of $\Hplane$ with $b
= 0$, hence $\Hplane = \Hplane(\unit_{\d}, 0)$.

Let define $\C^+$ the Hyper-cone of base $\K \cap \Hplane$, apex $Z$ such that
$\C^+ \subset \Hp$ and $\vol(\C^+) = \vol(\K^+)$. Moreover, let $\C^-$ the
extension of $\C^+$ such that $\C \doteq \C^- \cap \C^+$ is an Hyper-cone of
height $H$ and volume $\vol(\C) = \vol(\K)$. That is to say $\vol(\C^-) =
\vol(\K^-)$. From the same argument than before, we know that $\G(\C^+)$ (resp.
$\G(\C^-)$) is shifted with respect to $\G(\K^+)$ (resp. $\G(\K^-)$), thus,
according to the notation defined in Note \ref{note:realPoints} we have that
\[
\G(\K) \leq \G(\C)
\]

If $\x \leq \G(\C)$ then the exact same argument than the one of Section
\ref{sec:grunbaum} applies and
\[
\vol(\K^+) \geq \vol(\K) e^{-1} \tag{See Th. \ref{th:magic2} for details}
\]

Otherwise, we have that

\[
\G(\K) \leq \G(\C) \leq \x
\]

The idea of the proof is to find $\widetilde{\x}$ such that $\x
\leq \widetilde{\x}$ from which we can bound the volume of $\K^+$ in a similar
way than before.

Let define
\begin{equation}
\widetilde{\x} \doteq \G(\C) + \lambda H \left[ 1 - \frac{1}{n + 1} \right]
\unit_{n+1}
\end{equation}

Denote by $\Ball_0 = \Ball(B_0, R)$ the base of $\C$ and remind that $\C$ has
height $H$ and apex $Z$ and remind that $\cg(C) = B_0 + H/_{n+2} \unit{n+1}$.
Therefore
\[
\widetilde{\x} = H \left[ \lambda \left( 1 - \frac{1}{n+2} \right) +
\frac{1}{n+2} \right] \unit_{n+1}
\]

Consider $\widetilde{\Hplane} \doteq \Hplane(\unit_{n+1}, \widetilde{b})$ the
Hyperplane of normal vector $\unit_{n+1}$ (i.e. $\widetilde{\Hplane}$ is
parallel to $\Hplane$) and offset $\widetilde{b}$ such that $\widetilde{\x} \in
\widetilde{\Hplane}$. Then, let $\widetilde{\C^+}$ the positive partition of
$\C$ with respect to $\widetilde{\Hplane}$
\[
\widetilde{\C^+} = \C \cap \widetilde{\Hp}
\]

We can compute the volume of $\widetilde{\C^+}$ as
follow :
\begin{align*}
\vol(\widetilde{\C^+}) & = \int_{\widetilde{\x}}^{H} \vol \left[ \Ball \left(
B_0 + h \unit_{n+1}, \frac{H - h}{H}R \right) \right] dh \\
& = \int_{\widetilde{\x}}^{H}  \Cpi_n \frac{R^{n}}{H^{n}} (H - h)^{n} dh
\tag{Volume of a Ball} \\
& = \int^{H - H \left[ \lambda \left( 1 - \frac{1}{n+2} \right) +
\frac{1}{n+2} \right]}_{0} \Cpi_n \frac{R^{n}}{H^{n}} u^n du \tag{Subst. $u
\doteq H - h$} \\
& = \frac{\Cpi_n R^{n}}{H^{n}} \int^{H (1 - \lambda) \left[ 1 - \frac{1}{n+2}
\right]}_{0} u^n du \\
& = \frac{\Cpi_n R^{n}}{H^{n}} \times \frac{H^{n+1}}{n+1}
 (1 - \lambda)^{n+1} \left[1 - \frac{1}{n+2} \right]^{n+1} \\
& = \vol(\C) (1 - \lambda)^{n+1} \left[1 - \frac{1}{n+2} \right]^{n+1}
\tag{Volume of a Hyper-cone} \\
& \geq \vol(\C) (1 - \lambda)^{n+1} e^{-1} \tag{See Prop. \ref{prop:HconeRatio}}
\end{align*}

Where in the first two lines, we allow a slight abuse of notation and use
$\widetilde{\x}$ as a real as explained in Note \ref{note:realPoints}.

Then, we can rewrite the volume of $\C^+$ as 
\[
\vol(\C^+) = \int_{\x}^{\widetilde{\x}} \vol \left[ \Ball \left(
B_0 + h \unit_{n+1}, \frac{H - h}{H}R \right) \right] dh +
\vol(\widetilde{\C^+})
\]

Consequently, if $\x \leq \widetilde{\x}$ then the first term of $\vol(\C^+)$ is
positive and therefore, $\vol(\C^+) \geq \vol(\widetilde{\C+})$

\begin{align*}
\widetilde{\x} & = \G(\C) + \lambda H \left[ 1 - \frac{1}{n+2} \right] 
\\
& = \G(\C) + \lambda \frac{(n+1) \vol(\C)}{\Cpi_n R^{n}} \left[  1 -
\frac{1}{n+2} \right] \tag{Volume of a Hyper-cone} \\
& \geq \G(\K) + \lambda \frac{(n+1) \vol(\K)}{\Cpi_n R^{n}} \left[  1 -
\frac{1}{n+2} \right] \tag{$\G(\C) \geq \G(\K)$ and $\vol(\C) = \vol(\K)$}
\end{align*}

Unfortunately, we cannot easily compute $R$ directly. Nonetheless, since
$\Ball_0$ and $\Hplane$ are parallel, we can use the Trianngle proportionality
theorem. Denote $R_{\C^+}$ the radius of the base of $\C^+$, that is $\K \cap
\Hplane$ and $H_{\C^+}$ the height of $\C^+$ (i.e. the distance between $\x$
and $Z$) then we have :
\[
\frac{1}{R} = \frac{H_{\C^+}}{H R_{\C^+}}
\]

From this, we want to bound $H_{\C^+}$ and $H$ since $R_{\C^+}$ is easy enough
to estimate because it is directly related to $\K$ and $\Hplane$.

From previous argument, we know that $Z \notin \K^+$ except if $\K^+ = \C^+$. So
let define 
\[
H_{\K^+} \doteq \max_{a \in \K^+} a^T \unit_{n+1} 
\]
Intuitively, $H_{\K^+}$ is maximal distance between a point in $\K^+$ and $\x$
with respect to the axis of $\unit_{n+1}$. Because $Z$ is precisely on this
axis, and that $Z \notin \K^+$ (or $\K^+ = \C^+$) the following holds true 
\begin{equation} \label{eq:Hkp}
H_{\K^+} \leq H_{\C^+}
\end{equation}

Conversly, let define $H_{\K^-}$ as the maximal distance between $\x$ and any
 point of $\K^-$ with respect to the axis of $\unit_{n+1}$. Again, from
previous argument, we know that $B_0 \in \K^-$.
Note that, because $\C$ is a Hyper-cone, we know that $B_0$ is at a distance
$H/_{n+2}$ of $\G(\C)$. Moreover, remind that we are treating the case where
$\x \geq \G(\C)$, hence, the distance between $B_0$ and $\x$ is at least
$H/_{n+2}$ which in turn is smaller than $H_{\K^-}$. Reordering gives the
following

\begin{align}
 H_{\K^-} & \geq \frac{H}{n+2} \notag \\
\Leftrightarrow  (n + 2) H_{\K^-} & \geq H \label{eq:Hkm}
\end{align}

Putting back equations $\eqref{eq:Hkp}$ and $\eqref{eq:Hkm}$ together, we have
that
\[
\frac{1}{R} \geq \frac{H_{\K^+}}{(n+2)H_{\K^-} R_{\C^+}}
\]

Which we plug back into the previous calculation

\begin{align*}
\widetilde{\x} & \geq \G(\K) + \lambda \frac{(n+1) \vol(\K)}{\Cpi_n R^{n}} \left[  1 -
\frac{1}{n+2} \right] \\
& \geq \G(\K) + \lambda \frac{(n+1) \vol(\K)}{\Cpi_n R_{\C^+}^n}
\left[ \frac{H_{\K^+}}{(n+2) H_{\K^-}}
\right]^n \left[ 1 - \frac{1}{n+2} \right] \\
& = \x
\end{align*}

Remind that we drop $\unit_{n+1}$ in the above since we treat $\widetilde{\x}$
and $\x$ as real numbers (see Note \ref{note:realPoints}).

To conlude this proof by rewinding all together. Namely, 
\begin{itemize}
  \item $\x \leq \G(\C)$ and 
  \[
  \vol(\K^+) \geq \vol(\K) e^{-1}
  \]
  \item $\x \geq \G(\C)$ and we can define $\widetilde(\x)$ such that
  \[
  \vol(\K^+) \geq \vol(\widetilde{\C^+}) \geq \vol(\K) (1 - \lambda)^n e^{-1}
  \]
\end{itemize}

Once again, we lift Assumption \ref{ass:Sym} as before by noting that spherical
symetry preserves volumes. One difference though lies in the fact that computing
$R_{\C^+}$ is no longer immediate in the general case. Notwithstanding, it can
be easily approximated within satisfactory precision.

As a final note, we may mention that distinguishing between these two cases is
non-trivial. Hence, without additionnal computation, only the worst bound can be
guaranteed.
\end{proof}

\end{document}